\documentclass[hidelinks,reqno,11pt]{article}
\usepackage{fullpage,amsmath,amsthm,amssymb,enumerate,enumitem,mathtools,float,accents,hyperref,url,graphicx,bm,algorithm, algorithmic}
\usepackage{tikz,pgfplots}

\numberwithin{equation}{section}

\newtheorem{theorem}{Theorem}[section]
\newtheorem*{theorem*}{Theorem}
\newtheorem{lemma}[theorem]{Lemma}
\newtheorem*{lemma*}{Lemma}

\newtheorem{corollary}[theorem]{Corollary}

\theoremstyle{definition}
\newtheorem{definition}[theorem]{Definition}
\newtheorem{remark}[theorem]{Remark}

\newtheorem{assumption}[theorem]{Assumption}

\DeclarePairedDelimiter{\abs}{\lvert}{\rvert}
\DeclarePairedDelimiter{\norm}{\lVert}{\rVert}
\DeclarePairedDelimiter{\paren}{(}{)}
\DeclarePairedDelimiter{\braces}{\lbrace}{\rbrace}
\DeclarePairedDelimiter{\inprod}{\langle}{\rangle}
\DeclarePairedDelimiter{\sqbracket}{[}{]}
\DeclarePairedDelimiter{\floor}{\lfloor}{\rfloor}
\DeclarePairedDelimiter{\tripnorm}{\lvert\kern-0.25ex\lvert\kern-0.25ex\lvert}{\rvert\kern-0.25ex\rvert\kern-0.25ex\rvert}
\DeclarePairedDelimiterX{\helperkldiv}[2]{(}{)}{#1\;\delimsize\|\;#2}

\renewcommand{\P}{\mathbb{P}}
\newcommand{\Var}{\text{Var}}

\newcommand{\E}{\mathbb{E}}
\newcommand{\R}{\mathbb{R}}
\newcommand{\C}{\mathbb{C}}

\newcommand{\bolda}{\textbf{a}}
\newcommand{\bolde}{\textbf{e}}

\newcommand{\boldx}{\textbf{x}}
\newcommand{\boldy}{\textbf{y}}

\newcommand{\boldr}{\textbf{r}}

\newcommand{\boldSigma}{\boldsymbol{\Sigma}}

\newcommand{\boldP}{\textbf{P}}

\newcommand{\boldI}{\textbf{I}}

\newcommand{\cross}{\times}

\newcommand{\grad}{\nabla}

\newlength{\dhatheight}

\newcommand{\sign}{\textnormal{sign}}

\newcommand{\Unif}{\textnormal{Unif}}

\title{Online Stochastic Gradient Descent with Arbitrary Initialization Solves Non-smooth, Non-convex Phase Retrieval}

\author{Yan Shuo Tan \footnote{Department of Statistics, University of California, Berkeley, \href{mailto:yanshuo@berkeley.edu}{yanshuo@berkeley.edu}.} \quad\quad\quad Roman Vershynin \footnote{Department of Mathematics, University of California, Irvine, \href{mailto:rvershyn@uci.edu}{rvershyn@uci.edu}.}}

\begin{document}

\maketitle

\begin{abstract}
	In recent literature, a general two step procedure has been formulated for solving the problem of phase retrieval. First, a spectral technique is used to obtain a constant-error initial estimate, following which, the estimate is refined to arbitrary precision by first-order optimization of a non-convex loss function. Numerical experiments, however, seem to suggest that simply running the iterative schemes from a random initialization may also lead to convergence, albeit at the cost of slightly higher sample complexity. In this paper, we prove that, in fact, constant step size online stochastic gradient descent (SGD) converges from arbitrary initializations for the non-smooth, non-convex amplitude squared loss objective. In this setting, online SGD is also equivalent to the randomized Kaczmarz algorithm from numerical analysis. Our analysis can easily be generalized to other single index models. It also makes use of new ideas from stochastic process theory, including the notion of a summary state space, which we believe will be of use for the broader field of non-convex optimization.
\end{abstract}

\section{Introduction}

The mathematical phase retrieval problem is that of recovering a $d$-dimensional signal vector $\boldx^* \in \R^d$ (or $\C^d$) from the magnitudes $b^{(k)} = \abs{\inprod{\bolda^{(k)},\boldx^*}}$ ($k=1,2\ldots,N$) of its projections onto a collection of known sampling vectors $\bolda^{(1)},\bolda^{(2)},\ldots,\bolda^{(N)} \in \R^d$ (or $\C^d$). Clearly, one can only hope to recover $\boldx^*$ up to global phase shift (i.e. up to multiplication by $e^{i\phi}$ for $\phi \in [0,2\pi)$), but mild regularity assumptions on the sampling vectors and $N$ large enough ensures that this is the only ambiguity.

This problem is well motivated by practical concerns, having applications to Coherent Diffraction Imaging (CDI), Electron Microscopy, and X-ray Crystallography, and as such has been a topic of study from at least the early 1980s. We refer the reader to the survey papers \cite{Fienup1982,Shechtman2014,Bendory2017} for a comprehensive account of the contexts in which the problem arises, as well as the techniques that practitioners employ to solve it.

Over the last decade, the phase retrieval problem has also garnered substantial attention from the optimization and machine learning communities, the reason being that it can be formulated as a relatively benign non-convex optimization problem. In other words, it can be solved by minimizing the least squares objective:
\begin{equation} \label{eq:intensities_squared_loss}
f(\boldx) = \frac{1}{2N}\sum_{k=1}^N \paren*{\abs{\inprod{\bolda^{(k)},\boldx}}^2-(b^{(k)})^2}^2.
\end{equation}

Many papers have attempted to study how to optimize this objective given distributional assumptions on the sampling vectors . One popular approach is to use a two-step procedure: First, a spectral technique is used to obtain an initial estimate $\boldx^{(0)}$ so that its distance from a global minimum, $\norm{\boldx^{(0)}-\boldx^*}_2$, is bounded above by a small constant. Next, the estimate is refined to arbitrary precision using an iterative method such as gradient descent or stochastic gradient descent (SGD). This procedure is well-supported by theoretical guarantees for both real and complex signals. Here we note that one can also construct loss functions for phase retrieval which are different from \eqref{eq:amplitudes_squared_loss}. Running variations of gradient descent or SGD on these functions also leads to provable guarantees (see for instance \cite{Candes2015,Zhang2016, Wang2017, Wei2015, Jeong2017, Tan2017}).

Nonetheless, this state of affairs is not entirely satisfactory from an optimization theory perspective, since the use of a spectral initialization diminishes the novelty of being able to provably minimize a non-convex objective using first-order methods. The spectral initialization essentially allows the first-order method to begin within a ``basin of convexity'', thus artificially escaping the difficulties of non-convexity. Such a \emph{deus ex machina} may not be available when trying to optimize other non-convex functions.

Consequently, it is important to investigate whether first order methods converge to $\boldx^*$ from a random or even arbitrary initialization. From here on, we work with only real signals and sampling vectors, i.e. $\boldx^*, \bolda^{(k)} \in \R^d$. In \cite{Sun2016}, the authors performed numerical experiments to show that this was indeed the case. They also analyzed the landscape of the loss function \eqref{eq:intensities_squared_loss} in the real setting, and showed that it enjoyed favorable qualities: given $N = \Omega(d\cdot\textnormal{polylog}(d))$ measurements, there are no spurious minima and all saddle points are proper. Therefore, saddle-point escaping methods like perturbed gradient descent enjoy polynomial time convergence.

More recently, the authors of \cite{Chen2018} showed that vanilla gradient descent converges in $O(\log d)$ iterations, again given $N = \Omega(d\cdot\textnormal{polylog}(d))$ measurements. Up to log factors, this matches the running time guarantees for the original two-stage method. While an important step, their analysis still requires full gradient updates and does not apply to SGD. In the high-dimensional setting, SGD is particularly advantageous because it can be applied in an online, streaming fashion. This lowers the space complexity of the algorithm from $O(Nd)$ to $O(d)$, and allows progress toward the solution to be made even before the analyst gains access to the full data sample.

Moreover, with respect to the non-smooth \emph{amplitude least squares} objective
\begin{equation} \label{eq:amplitudes_squared_loss}
f(\boldx) = \frac{1}{2N}\sum_{k=1}^N \paren*{\abs{\inprod{\bolda^{(k)},\boldx}}-b^{(k)}}^2,
\end{equation}
the question of even gradient descent convergence remains open. This objective is especially interesting because numerical simulations have shown gradient descent and SGD with respect to it to succeed with fewer measurements than are necessary for the alternative objective \eqref{eq:intensities_squared_loss} (See \cite{Wei2015} and \cite{Wang2017}.)

\subsection{Main Results}

In this paper, we prove that for a real signal vector and real sampling vectors, online stochastic gradient descent for the non-smooth objective \eqref{eq:amplitudes_squared_loss} converges to a global minimum from arbitrary initializations given $\Omega(d\log d)$ Gaussian measurements. We believe our work to be among the first in establishing convergence of SGD in the non-smooth, non-convex regime. Furthermore, it will be readily apparent that our analysis framework generalizes easily to other single index models, and we conjecture that similar techniques will also work low-rank matrix sensing models in general.

We perform SGD with respect to \eqref{eq:amplitudes_squared_loss}, using a single data point per iteration. More formally, we form a sequence of signal estimates $\boldx^{(0)}, \boldx^{(1)}, \boldx^{(2)},\ldots$ with the update rule:
\begin{equation} \label{eq:SGD_update}
\boldx^{(k+1)} \coloneqq \boldx^{(k)} + \eta \paren*{\sign\paren{\inprod{\bolda^{(k)},\boldx^{(k)}}} b^{(k)} - \inprod{\bolda^{(k)},\boldx^{(k)}}}  \bolda^{(k)}.
\end{equation}

We typically choose $\eta = \frac{1}{d}$. This is the same as in previous work that analyzed SGD as part of the two-step approach (see \cite{Tan2017}), and allows for a clear geometric interpretation. At each step $k$, we receive the datum $\paren*{\bolda^{(k)},b^{(k)}}$; the solution set to the corresponding equation $\abs{\inprod{\bolda^{(k)},\boldx}} = b^{(k)}$ is then the union of two parallel hyperplanes. Taking an SGD step projects the current iterate $\boldx^{(k-1)}$ onto the closer of these hyperplanes. This iterative projection is strongly reminiscent of the randomized Kaczmarz algorithm for solving linear systems. For a more in-depth discussion of this connection, we again refer the reader to \cite{Tan2017}. 

We make the following assumptions for the rest of the paper.

\begin{assumption} \label{assumptions}
	Assume the following for each positive integer $k$:
	\begin{enumerate}
		\item[(A)] (Fresh measurements) At step $k$ of the algorithm, we use a sampling vector $\bolda^{(k)}$ that is fully independent of the previous measurements $\bolda^{(1)},\ldots,\bolda^{(k-1)}$.
		\item[(B)] (``Gaussian'' measurements) We have $\bolda^{(k)} \sim \sqrt{d}\mathbb{S}^{d-1}$, where $\mathbb{S}^{d-1}$ is the unit sphere in $\R^d$.
		\item[(C)] (No noise) We have $b_k = \abs{\inprod{\bolda^{(k)},\boldx^*}}$.
	\end{enumerate}
\end{assumption}

Note that because of concentration of norm, Assumption B is almost equivalent to requiring that each $\bolda_k$ has a standard Gaussian distribution in $\R^d$. We will use this distributional assumption instead more for the sake of convenience of analysis rather than anything else.

The following is the main result of the paper.

\begin{theorem}[Main result] \label{thm:main_theorem}
	Suppose we run the iterative update \eqref{eq:SGD_update} in $\R^d$ under the assumptions above, and with constant step size $\eta = \frac{\eta_0}{d}$ where $\eta_0$ is a small enough universal constant. Then with probability at least $0.8 - 1/d - C/\log d$, there is a stopping time $T \lesssim d\cdot\paren*{\log d + \log \paren*{\frac{\norm{\boldx^{(0)}}}{\norm{\boldx^*}}\vee 1}}$, such that for all $k \geq T$, we have
	\begin{equation} \label{eq:main_theorem}
	\norm{\boldx^{(k)}-\sigma\boldx^*}^2 \leq \paren*{1-\frac{1}{2d}}^{k-T} \norm{\boldx^*}^2,
	\end{equation}
	where $\sigma = \sign(\inprod{\boldx^{(k)},\boldx^*})$. Furthermore, there is some constant $D$, such that if $d \geq D$, we may choose $\eta_0=1$.
\end{theorem}

This theorem tells us that in $T$ steps, SGD brings us to a point in the ``basin of convexity'' around a global minimum, within which we get linear convergence. As a consequence of this theorem, it is easy to see that we can get an $\epsilon$-relative-error estimate using $O(d \log(d/\epsilon))$ measurements (and the same number of steps). In fact, the analysis in \cite{Tan2017} tells us that within the ``basin of convexity'', linear convergence still holds if we resample from amongst $O(d)$ measurements. By doing this instead, our final sample complexity is $O(d \log d)$.

The novelty of this result is not simply the convergence of the algorithm, but rather its convergence with \emph{near optimal sample and time complexity}. Indeed, it is well known that as we let the step size decrease to zero, SGD under our assumptions approximates gradient flow for the population loss function. It is easy to show that gradient flow converges to a global minimum. With smaller steps, however, the algorithm takes a longer time to converge.

In addition, our chosen step size $\frac{1}{d}$ is the smoothness parameter for SGD if we were optimizing a least squares system objective under the same assumptions on $\bolda^{(k)}$. It is interesting that although our objective is no longer smooth, $\frac{1}{d}$ nonetheless remains the right scaling.

The proof of the result is somewhat complicated and makes use of several new ideas. The first idea is to think of the sequence of SGD iterates as a Markov chain on a two-dimensional \emph{summary state space} $\mathcal{Y}$. The two coordinates are the squared Euclidean norm of the iterate, $r^2 = \norm{\boldx}^2$, and the correlation with the signal, $s = \inprod{\boldx,\frac{\boldx^*}{\norm{\boldx^*}}}$. Just as how in thermodynamics state variables such as temperature, pressure, and volume suffice to determine the evolution of a thermodynamic system, so too in our case do the state variables $r^2$ and $s$ suffice to determine the progress of the SGD algorithm.

The state space is obviously independent of the dimension $d$. In fact, one can show that the distribution of the update in the state space is effectively independent of the dimension up to overall scaling (see Theorem \ref{thm:state_space_MC}). As such, as $d$ tends to infinity, the stochastic dynamics of the Markov chain when initialized at a \emph{fixed} $\boldy \in \mathcal{Y}$ approximates the solution of an ODE system whose corresponding vector field is given by the rescaled drift of the process.

Unfortunately, we are hit with the curse of dimensionality: if we take a random initialization $\boldx^{(0)} \sim \mathcal{N}(0,\boldI)$, then it is well-known that with high probability,
\begin{equation} \label{eq:curse_of_dimensionality}
s_0 = \frac{\abs{\inprod{\boldx^{(0)},\boldx^*}}}{\norm{\boldx^{(0)}}\norm{\boldx^*}} \lesssim \frac{1}{\sqrt{d}}.
\end{equation} 
This implies that the initial correlation with the signal decays as the dimension increases, and in fact, at the corresponding point $(r_0^2,s_0)$ in the state space, the drift and the fluctuations have the same magnitude, and it is no longer appropriate to approximate the stochastic dynamics with a deterministic process.

Overcoming this is the most difficult part of the proof. To do so, we use a small-ball probability argument. In other words, we show that the distribution of $s_K = s(\boldx^{(K)})$ is anti-concentrated away from 0 when $K$ is large enough. This involves comparing the process $s_0,s_1,\ldots$ with a more well-understood process $\hat{s}_0,\hat{s}_1,\ldots$ via stochastic dominance. The distribution of this new sequence can in turn be controlled via recursive inequalities bounding 4th moments from above and 2nd moments from below. We conclude by applying the Paley-Zygmund inequality.

\subsection{Related Work}

\subsubsection{Phase retrieval}

There is already a large body of work on phase retrieval, and it is impossible to give a full account of the literature. We have already mentioned survey papers for how phase retrieval arises in various engineering problems. On the theoretical side, we have already discussed the two-step non-convex optimization approach, and will further mention here the convex relaxation approaches pioneered in the papers \cite{Cand??s2013,Goldstein2016,Bahmani2016a,Hand2016a}.

\subsubsection{SGD as a Markov chain}

It has long be observed that constant step-size SGD can be thought of as a Markov chain, and there seems to be a resurgence of interest in this view of SGD. For instance, \cite{Dieuleveut2017} uses this approach to analyze the limiting distribution of SGD iterates for strongly convex functions. Furthermore, \cite{Mandt2017} uses this interpretation to see how SGD can be used as a sampling algorithm. Both these works have drawn inspiration from the recent body of work on Langevin algorithms for sampling from log-concave distributions. The idea of analyzing SGD through diffusion approximation is also present in \cite{Li2016}.

\subsubsection{Non-convex optimization and first-order methods}

The Kaczmarz method is a classical method in numerical analysis for solving large scale overdetermined linear systems. A randomized version of it was first analyzed by \cite{Strohmer2009}. In our earlier work \cite{Tan2017}, we proposed adapting the method to the setting of phase retrieval, where it coincides with SGD under the Gaussian measurement setting,

Stochastic first-order methods have emerged as the optimization method of choice for modern machine learning. In particular, deep neural networks are trained almost exclusively using SGD and variants like ADAM. The loss functions for these models, however, are non-convex functions, for which there has traditionally been little theory on how first-order methods behave.

Unsurprisingly, there has been a concerted push over the last few years to address this issue. One line of work studies how gradient descent or SGD can be made to escape saddle points quickly (see \cite{Ge2015,Jin2017,Jin2019}). Another line of work has focused on identifying regimes for shallow and deep neural networks for which gradient descent or SGD can be shown to converge (see for instance \cite{Mei2018,Allen-Zhu2018}).

\subsection{Notation}

Scalars are denoted in normal font, whereas vectors and matrices are denoted in bold. Subscripts (usually) denote components of a vector, while superscripts denote the index of a quantity when it is part of a sequence. Sets and events are denoted using calligraphic font. The indicator of a set $\mathcal{A}$ is denoted by $\boldsymbol{1}_{\mathcal{A}}$. The Euclidean norm is denoted with subscript omitted: $\norm{-}$. No other norms are used in this paper, so there is no risk of confusion. When $X$ is a subexponential random variable, $\norm{X}_{\psi_1}$ denotes its subexponential norm (for a definition, see Appendix \ref{sec:properties_of_subexponentials}). We use $\mathcal{B}(\R^d)$ to denote the Borel $\sigma$-algebra for $\R^d$. Sequences of the form $x_0, x_1, x_2,\ldots$ are denoted by $\braces{x_k}_k$. Throughout the paper, $C$, $c$, $C_1$, $C_2$ and $C_3$ denote positive constants that may change from line to line.

\section{Outline of proof}

We start by making some simplifying assumptions. First, note that the algorithm and our guarantee \eqref{eq:main_theorem} are both invariant with respect to scaling and rotation. As such, we may assume without loss of generality that $\boldx^* = \bolde_1$, the first coordinate basis vector. We also only analyze the case where $d$ is large enough so that $\eta_0=1$ and the step size is set to be $\eta = \frac{1}{d}$. The extension to smaller $d$ will be obvious.

\subsection{State space Markov chain} \label{subsec:state_space_MC}

The state space is the set $\mathcal{Y} \coloneqq \braces*{\boldy = (r^2,s) \in \R^2 ~\colon s^2 \leq r^2}$, where $r$ is the Euclidean norm of the SGD iterate and $s$ is its projection onto the signal direction. In other words, the natural projection map $\pi\colon \R^d \to \mathcal{Y}$ is defined by
\begin{equation*}
r^2(\boldx) \coloneqq \norm{\boldx}^2 \quad\quad\textnormal{and}\quad\quad s(\boldx) \coloneqq \inprod{\boldx,\boldx^*}.
\end{equation*}
The reason we choose to use $r^2$ instead of $r$ is for the convenience of obtaining formulas for the stochastic update, as will be evident later. We further define $\theta = \theta(\boldx) \coloneqq \arccos(s/r)$. This is the smaller angle between $\boldx$ and $\boldx^*$.

Note that we can track the progress of SGD purely in terms of the state variables. Indeed, we have
\begin{align} \label{eq:definition_for_Psi}
\norm*{\boldx-\sign(\inprod{\boldx,\boldx^*})\boldx^*}^2 & = \norm*{\boldx}^2 - 2\abs*{\inprod{\boldx,\boldx^*}} + \norm*{\boldx^*}^2 \nonumber \\
& = r^2 - 2\abs{s} + 1 \nonumber \\
& \eqqcolon \Psi(r^2,s),
\end{align}
so that the error of the $k$-th step estimate $\boldx^{(k)}$ is equal to $\Psi(\pi(\boldx^{(k)}))$. Note that $-\boldx^*$ and $\boldx^*$ are mapped onto $(1,-1)$ and $(1,1)$, so that $\Psi$ is uniquely minimized at these values. We hence wish to show that $r^2$ and $s$ coordinates of our iterates converge to $1$ and $\pm 1$ respectively.

It is clear that the SGD sequence $\boldx^{(0)},\boldx^{(1)}, \boldx^{(2)},\ldots$ is a Markov chain on $(\R^d,\mathcal{B}(\R^d))$, with the update rule \eqref{eq:SGD_update} giving a random mapping representation for the transition kernel. Now, consider the sequence $\boldy^{(0)}, \boldy^{(1)}, \boldy^{(2)},\ldots$ where we define $\boldy^{(k)} \coloneqq \pi(\boldx^{(k)})$ for each index $k$. We are now ready to state the first key insight:

\begin{theorem} \label{thm:state_space_MC}
	The sequence $\boldy^{(0)}, \boldy^{(1)}, \boldy^{(2)},\ldots$ is a Markov chain on $(\mathcal{Y},\mathcal{B}(\mathcal{Y}))$ whose transition kernel has the random mapping representation \begin{equation} \label{eq:state_space_update}
	\boldy^{(k+1)} = \boldy^{(k)} + \frac{1}{d}\paren*{\alpha(\boldy^{(k)}),\beta(\boldy^{(k)})},
	\end{equation}
	where
	\begin{equation} \label{eq:formula_for_alpha}
	\alpha(r^2,s) \coloneqq (1-r^2\cos^2\theta) \cdot u^2-r^2\sin^2\theta \cdot v^2 +r^2\sin\theta\cos\theta \cdot uv,
	\end{equation}
	\begin{equation} \label{eq:formula_for_beta}
	\beta(r^2,s) \coloneqq (1-r\cos\theta- 2\boldsymbol{1}_{\mathcal{A}}) \cdot u^2 - r\sin\theta \cdot uv.
	\end{equation}
	Here, the randomness is supplied by $(u,v)$, which is a 2-dimensional marginal of the uniform distribution on $\sqrt{d}\mathbb{S}^{d-1}$, while $\mathcal{A} = \mathcal{A}(\theta)$ is the event that $\sign(\cos\theta u + \sin\theta v) \neq \sign(u)$.
\end{theorem}

\begin{proof}
	Deferred to Appendix \ref{sec:state_space_calculations}.
\end{proof}

This theorem tells us that the state space sequence $\boldy^{(0)}, \boldy^{(1)}, \boldy^{(2)},\ldots$ suffices not just to track our progress, as discussed earlier in the section, but also to determine its own dynamics. We hence no longer need to concern ourselves with the original SGD sequence, and instead work with this object for the rest of the paper. Henceforth, we let $\braces{\mathcal{F}_k}_k$ denote the filtration defined by this sequence.

\subsection{Doob decomposition and continuous time limit as $d \to \infty$} \label{subsec:Doob_decomposition}

Let us try to understand the random mappings \eqref{eq:formula_for_alpha} and \eqref{eq:formula_for_beta} better. It is well-known that $(u,v)$ converges in distribution to a standard 2-dimensional Gaussian $\mathcal{N}(0,\boldI_2)$ as the ambient dimension $d$ tends to infinity. Therefore, the only essential dependence of the update formula \eqref{eq:state_space_update} on $d$ is through the overall $\frac{1}{d}$ scaling. If we think of the indices $k=1,2,\ldots$ as a time variable, rescale time by a factor of $\frac{1}{d}$, we can think of the sequence as being generated by an Euler discretization of a continuous-time process.

While we do not actually take this approach in our rigorous analysis, is it instructive to see what intuition this gives us. To do this, we do a Doob decomposition of the process $\braces*{\boldy^{(k)}}_{k=0}^\infty$, separating it into a drift term and a fluctuation term. Denote the drift terms using
\begin{equation*}
\bar{\alpha}(\boldy) \coloneqq \E\braces{\alpha(\boldy)} \quad\quad\textnormal{and}\quad\quad \bar{\beta} \coloneqq \E\braces{\beta(\boldy)}.
\end{equation*}
Letting $(\alpha_j,\beta_j)$ denote the random mapping used in the $j$-th step of the Markov chain, we have
\begin{equation} \label{eq:Doob_decomposition}
\boldy^{(k)} - \boldy^{(0)} = \underbrace{\frac{1}{d}\sum_{j=1}^k \paren*{\bar{\alpha}(\boldy^{(j-1)}), \bar{\beta}(\boldy^{(j-1)}) }}_{\textnormal{drift}} + \underbrace{\frac{1}{d}\sum_{j=1}^k \paren*{\alpha_j(\boldy^{(j-1)})-\bar{\alpha}(\boldy^{(j-1)}), \beta_j(\boldy^{(j-1)}) - \bar{\beta}(\boldy^{(j-1)}) }}_{\textnormal{fluctuation}}.
\end{equation}

We now try to do a heuristic comparison of the relative magnitudes of the two terms. Suppose $k$ is small enough so that we have $\boldy^{(j)} \approx \boldy^{(0)}$ for $j=1,\ldots,k$. Then the drift can be approximated by
\begin{equation*}
\frac{1}{d}\sum_{j=1}^k \paren*{\bar{\alpha}(\boldy^{(j-1)}), \bar{\beta}(\boldy^{(j-1)}) } \approx \frac{k}{d}\paren*{\bar{\alpha}(\boldy^{(0)}),\bar{\beta}(\boldy^{(0)})  }.
\end{equation*}
Meanwhile, we also have
\begin{equation*}
\frac{1}{d}\sum_{j=1}^k \paren*{\alpha_j(\boldy^{(j-1)})-\bar{\alpha}(\boldy^{(j-1)}), \beta_j(\boldy^{(j-1)}) - \bar{\beta}(\boldy^{(j-1)}) } \approx \frac{1}{d}\sum_{j=1}^k \paren*{\alpha_j(\boldy^{(0)})-\bar{\alpha}(\boldy^{(0)}), \beta_j(\boldy^{(0)}) - \bar{\beta}(\boldy^{(0)}) },
\end{equation*}
so that the fluctuation term has standard deviation approximately equal to
\begin{equation*}
\frac{\sqrt{k}}{d}\paren*{\Var\braces{\alpha(\boldy^{(0)})}^{1/2},\Var\braces{\beta(\boldy^{(0)})}^{1/2}}.
\end{equation*}

Therefore, for any \emph{fixed} $\boldy^{(0)}$, we see that the drift dominates the fluctuations as $d$ tends to infinity. This means that the continuous time limit of the process trajectory should be an \emph{integral curve} associated to the vector field on the state space $\mathcal{Y}$ defined by $\paren*{\bar{\alpha},\bar{\beta}}$. While this picture is incomplete, it offers a good first approximation, and the next step we take is to analyze the solutions to this first order ODE system.

\subsection{Drift in continuous time limit}

Miraculously, it is actually possible to derive a closed form formula for the vector field. We state it in the following lemma.

\begin{lemma}[Formula for drift] \label{lem:SGD_drift}
	With the notation $\bar{\alpha}(\boldy) \coloneqq \E\braces{\alpha(\boldy)}$ and $\bar{\beta} \coloneqq \E\braces{\beta(\boldy)}$, we have
	\begin{equation} \label{eq:expectation_of_alpha}
	\bar{\alpha}(r^2,s) = 1-r^2,
	\end{equation}
	\begin{equation} \label{eq:expectation_of_beta}
	\bar{\beta}(r^2,s) = 1 - s - \frac{1}{\pi}\paren*{2\theta-\sin(2\theta)}.
	\end{equation}
\end{lemma}

\begin{proof}
	Deferred to Appendix \ref{sec:state_space_calculations}.
\end{proof}

Studying the vector field plot in Figure \ref{fig:vector_field}, it is obvious that $\boldy^* \coloneqq (1,1)$ and $-\boldy^* = (-1,1)$ are the only attracting fixed points, with basins of attraction the sets $\mathcal{Y}_+ \coloneqq \mathcal{Y}\cap \braces{s > 0}$ and $\mathcal{Y}_- \coloneqq \mathcal{Y}\cap \braces{s < 0}$ respectively. While this assures us that the system has the right qualitative long-term behavior, the visualization alone is not sufficient to give quantitative bounds on convergence rates. This analysis turns out to be somewhat tricky. Given an integral curve $\bar{\boldy}^{(t)} = (\bar{r}_t^2,\bar{s}_t)$ starting from an arbitrary initialization $\bar{\boldy}^{(0)}$, we will analyze its convergence rate by breaking it into three separate phases, as depicted in the figure.

\begin{figure}[h]
	\includegraphics[scale=0.6]{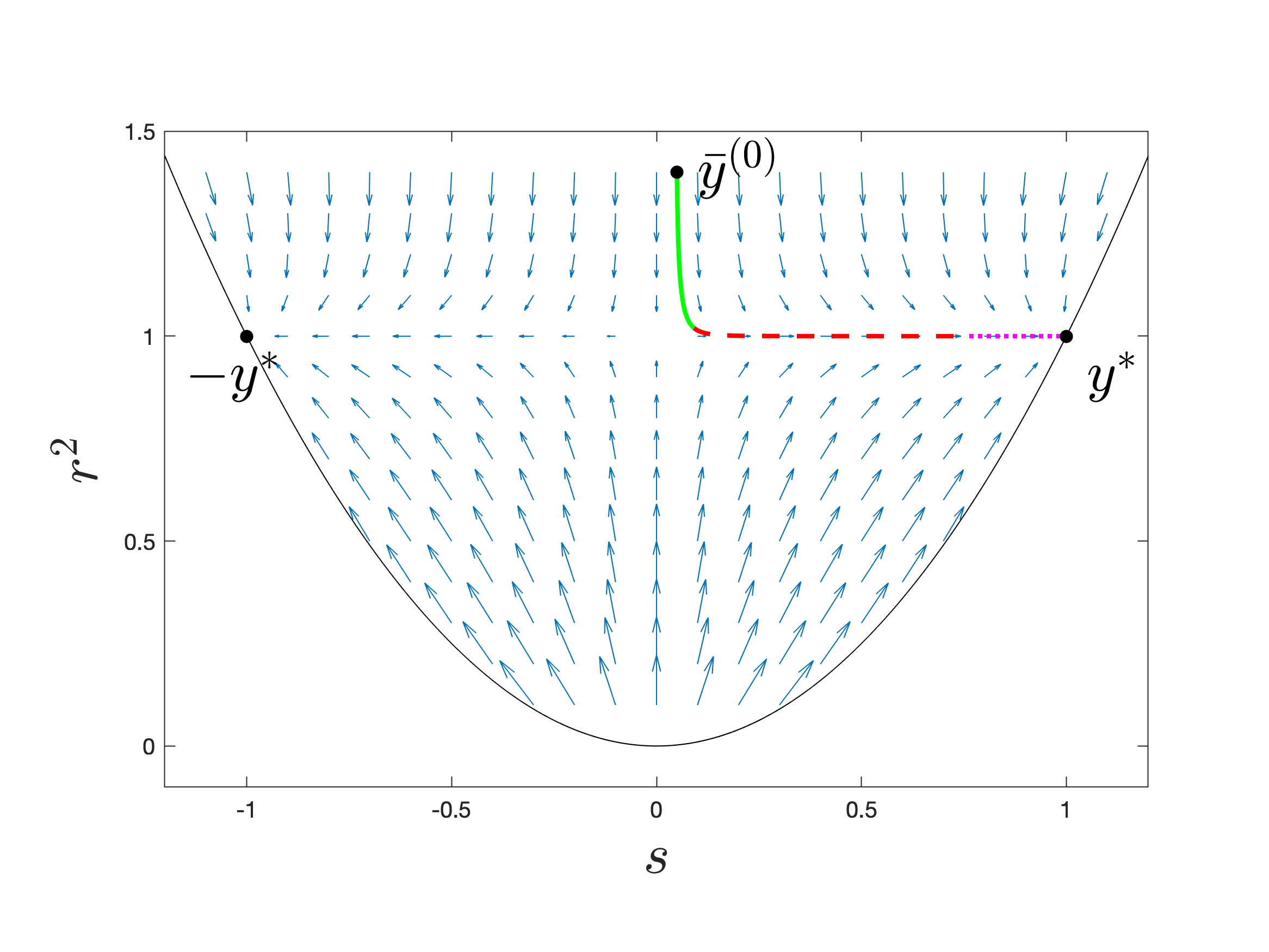}
	\centering
	\caption{Vector field defined by $\paren*{\bar{\alpha},\bar{\beta}}$. For the purposes of symmetry, we have chosen to plot $r^2$ along the vertical axis and $s$ along the horizontal axis. An integral curve from an initialization $\boldy^{(0)}$ is plotted. Our analysis for the convergence rate will be broken into three phases. Phase 1 concerns the portion of the curve colored in green, Phase 2 concerns that colored in red, and Phase 3 concerns the last portion colored in magenta.}
	\label{fig:vector_field}
\end{figure}

To demarcate the phases, we define two stopping times as follows. We let $\bar{\tau}_1$ be the earliest time $t$ for which $\abs*{\bar{r}_t^2-1} \leq 0.1$, and we let $\bar{\tau}_2$ be the earliest time $t$ for which the Lyapunov function $\Psi$ defined in \eqref{eq:definition_for_Psi} satisfies $\Psi(\bar{\boldy}^{(t)}) \leq 0.2$. Phase 1 is then the portion of the curve traversed between time $0$ and time $\bar{\tau}_1$, Phase 2 the portion traversed between time $\bar{\tau}_1$ and time $\bar{\tau}_2$, with Phase 3 the remainder of the curve traversed after $\bar{\tau}_2$.

Let us compute the duration of Phase 1, which is the same as bounding $\bar{\tau}_1$. To do this, we solve \eqref{eq:expectation_of_alpha} to get $\bar{r}_t^2-1 = e^{-t}(\bar{r}_0^2 -1)$, so that $\bar{\tau}_1 \lesssim \log(\bar{r}_0^2)\vee 1$. The second phase is trickier due to the unwieldiness of \eqref{eq:expectation_of_beta}. As such, we compute more amenable lower and upper bounds for the expression as follows.

\begin{lemma}[Bounds for horizontal drift] \label{lem:lower_bound_for_drift}
	There is a constant $\bar{b}_{max}$ such that we have
	\begin{equation*}
	\sup_{r \geq 1/2} \frac{\bar{\beta}(r^2,s)}{s} \leq \bar{b}_{max}.
	\end{equation*}
	Furthermore, for any $\epsilon > 0$ small enough, there is some $\eta = \eta(\epsilon) > 0$ and some constant $\bar{b}_{min} = \bar{b}_{min}(\epsilon) > 0$ such that
	\begin{equation*}
	\inf_{(r^2,s)\in \mathcal{D}}\frac{\bar{\beta}(r^2,s)}{s} \geq \bar{b}_{min},
	\end{equation*}
	where $\mathcal{D} \coloneqq \braces{(r^2,s) \in \mathcal{Y} ~\colon~ \abs{s} \leq 1-\epsilon, \abs{r^2-1} \leq \eta}$.
\end{lemma}

\begin{proof}
	Deferred to Appendix \ref{sec:drift_lemmas}.
\end{proof}

One can show that $\eta(0.1) \geq 0.1$, and therefore, we have the bound $\frac{d\bar{s}_t}{dt} \geq \bar{b}_{min}\bar{s}$ for $\bar{\tau}_1 \leq t \leq \bar{\tau}_2$. Solving this gives $\bar{s}_t = \bar{s}_{\bar{\tau}_1}e^{\bar{b}_{min}(t-\bar{\tau}_1)}$, and we have the estimate $\bar{\tau}_2-\bar{\tau}_1 \lesssim \log(1/\abs{\bar{s}_{\bar{\tau}_1}}) \vee 1$.

Finally, Phase 3 corresponds to portion of the integral curve that lies within the ``basin of convexity'' around $\boldy^*$. Indeed we compute:

\begin{align*}
\frac{d}{dt}\Psi(\bar{\boldy}^{(t)}) & = \partial_{r^2}\Psi\cdot\bar{\alpha} + \partial_s\Psi\cdot\bar{\beta} \\
& = 1-\bar{r}_t^2 - 2\paren*{1 - \bar{s}_t - \frac{1}{\pi}\paren*{2\bar{\theta}_t-\sin(2\bar{\theta}_t)}} \\
& \lesssim - \paren*{\bar{r}_t^2-2\bar{s}_t + 1} \\
& = -\Psi(\bar{\boldy}^{(t)}).
\end{align*}
Here, the inequality in the third line comes from a relative bound on the error term $\frac{1}{\pi}\paren*{2\bar{\theta}_t-\sin(2\bar{\theta}_t)}$ provided by the geometry of the basin region. As such, we also get linear convergence $\Psi(\bar{\boldy}^{(t)}) \leq \Psi(\bar{\boldy}^{(\bar{\tau}_2)})e^{-c(t-\bar{\tau}_2)}$.

Putting everything together, we see that for any $\epsilon > 0$, if we would like $\Psi(\bar{\boldy}_t) \leq \epsilon$, it suffices for

\begin{equation} \label{eq:cts_time_convergence}
t \gtrsim \log(\bar{r}_0^2)\vee 1 + \log(1/\abs{\bar{s}_{\bar{\tau}_1}})\vee 1 + \log(1/\epsilon).
\end{equation}

\subsection{Discretizing the drift}

We now examine what this means for the Markov chain $(r^2_k,s_k) = \boldy^{(k)} = \boldy^{(k,d)}$, where for clarity, we have made the dependence on $d$ in \eqref{eq:state_space_update} explicit as a component of the indexing. We have argued that the integral curve $\bar{\boldy}^{(t)}$ is the limit of the trajectories $\braces{ t \mapsto \boldy^{(\floor{t/d},d)}}$ as $d$ tends to infinity. The number of iterations $k$ needed for convergence of $\boldy^{(k)}$ to $\pm\boldy^*$ is thus the quantity \eqref{eq:cts_time_convergence} scaled by a factor of $d$.

There may also some additional dependence on $d$ implicit in the definitions of $\bar{r}_0^2$ and $\bar{s}_{\bar{\tau}_1}$. The first has to do with the conditioning of the problem, and it is reasonable to assume that $\bar{r}_0^2 \lesssim 1$. On the other hand, if we take a random initialization, measure concentration inflicts us with the curse of dimensionality shown in \eqref{eq:curse_of_dimensionality}, so that $\log(1/\abs{\bar{s}_{\bar{\tau}_1}}) \gtrsim \log d$. The final time complexity estimate is therefore $O(d\log (d/\epsilon))$.

\subsection{Accounting for fluctuations}

Comparing the time complexity estimate for the drift process given in the last section with the guarantee given by Theorem \ref{thm:main_theorem} suggests to us that the fluctuations do not ultimately affect the rate of convergence for the random process $\boldy^{(k)}$. This is true, and indeed, we may make our heuristic arguments rigorous and show that the fluctuation ``error term'' is dominated by the drift over the discretized versions of Phase 1 and Phase 3.

The situation is more complicated for Phase 2. While the length of the phase for the random process remains the same, the underlying \emph{dynamics} of the random process is very different from that of the deterministic drift process. This is because when $\abs{s_k} \lesssim 1/\sqrt{d}$, as would be the case under random initialization, the magnitude of the fluctuations are of the same order as the drift, thereby invalidating the approximation argument.

There is no hope of upper bounding the fluctuations of $\boldy^{(k)}$ around iteration $k \approx d\bar{\tau}_1$, so we change course and instead aim at bounding the cumulative variance of the fluctuations \emph{from below}. We seek a small ball probability bound for the horizontal marginal $s_{d\bar{\tau}_1 + k'}$ for $k' \asymp d\log d$. More precisely, we would like
\begin{equation} \label{eq:small_ball_prob}
\lim_{d\to\infty} \P\braces*{\abs{s_{d\bar{\tau}_1+cd\log d}} \geq \eta } \geq  \delta 
\end{equation}
for some universal constants $c$, $\eta$ and $\delta$. Once $s_k$ is of a constant distance away from zero, we may return to using the approximation argument and treat the fluctuations as an error term.

\subsection{Small ball probability bound through diffusion approximation}

In order to obtain a small ball probability bound, we will need a better understanding of the distribution of the fluctuation term in \eqref{eq:Doob_decomposition}, rather than merely control over its tail probabilities. Fortunately, this term constitutes a martingale difference process, and the martingale Central Limit Theorem tells us that it converges to Brownian motion. In other words, we should expect the random process $\boldy^{(\floor{t/d},d)}$ to be well-approximated by a \emph{diffusion} process $\tilde{\boldy}^{(t)}$ satisfying the stochastic differential equation
\begin{equation*}
d\tilde{\boldy}^{(t)} = \paren*{\bar{\alpha}(\tilde{\boldy}^{(t)}),\bar{\beta}(\tilde{\boldy}^{(t)})} dt + \frac{\boldSigma(\tilde{\boldy}^{(t)})}{\sqrt{d}} d\textbf{B}^{(t)}.
\end{equation*}
Here, $d\textbf{B}^{(t)}$ is standard, two-dimensional Brownian motion, while for each $\boldy$, $\boldSigma(\boldy)$ is a positive semidefinite matrix reflecting the fluctuation covariance at state $\boldy$.

It is not easy to compute a closed form solution to this SDE. To perform a heuristic analysis, we instead solve a simplified form of the equation for the $s$-marginal:
\begin{equation*}
d\tilde{s}_t = b \tilde{s}_t dt + \frac{\sigma}{\sqrt{d}}dB_t.
\end{equation*}
This is a good approximation if we assume $\abs{\tilde{r}_t^2-1}, \abs{\tilde{s}_t} \ll 1$, so that $\boldSigma(\boldy)$ is approximately constant in $\boldy$, while $\bar{\beta}(r^2,s)$ is close to its linear approximation in $s$. Solving this equation gives
\begin{equation*}
\tilde{s}_t = \tilde{s}_0 e^{bt} + \frac{e^{bt}}{\sqrt{d}}\int_0^t e^{-b\tau} dB_\tau \sim \mathcal{N}\paren*{\tilde{s}_0 e^{bt},\frac{e^{2bt} - 1}{bd}}.
\end{equation*}
As such, no matter the value of $\tilde{s}_0$, we have that $\tilde{s}_{c\log d}$ is a Gaussian with variance bounded below by a constant, which implies the desired estimate \eqref{eq:small_ball_prob}.

\subsection{Outline for rest of paper}

In this section so far, we have sketched a heuristic proof for SGD convergence, where the main idea was to consider the continuous time limit of the state space Markov chain, and solve the resulting differential equation or stochastic differential equation. In our rigorous analysis, we will not adopt this approach, and instead solve finite difference equations coming from the Markov chain while obtaining non-asymptotic control over the fluctuation process.

For the rest of this paper, we work with a fixed initialization $\boldy^{(0)}$, and let $\boldy^{(1)}, \boldy^{(1)}, \boldy^{(2)},\ldots$ be the sequence of iterates generated by repeated applying the Markov kernel \eqref{eq:state_space_update}. For each $k$, we will use $r_k^2$ and $s_k$ to denote the coordinates of $\boldy^{(k)}$. We will continue to do a multi-phase analysis of convergence, and as such, define analogues of the stopping times $\bar{\tau}_1$ and $\bar{\tau}_2$. We set
\begin{equation} \label{eq:tau_1}
\tau_1 \coloneqq \min\braces{k \colon \abs{r_k^2-1} \leq \log d/\sqrt{d}}
\end{equation}
\begin{equation} \label{eq:tau_2a}
\tau_{2a} \coloneqq \min\braces{k \geq \tau_1 \colon \abs{s_k} \geq \gamma_1}
\end{equation}
\begin{equation} \label{eq:tau_2b}
\tau_{2b} \coloneqq \min\braces{k \geq \tau_{2a} \colon \Psi(\boldy^{(k)})} \leq \gamma_2.
\end{equation}
Here, $\gamma_1$ and $\gamma_2$ are constants to be determined later. 

We shall set $T = \tau_{2b}$ in Theorem \ref{thm:main_theorem}. As such, we wish to prove that with high probability, $\tau_{2b} \lesssim d\log\paren*{d\norm{\boldx^{(0)}}\vee 1}$, and that linear convergence in expectation occurs after $\tau_{2b}$. The second statement follows easily from the theory we established in \cite{Tan2017}, while the first is, as mentioned before, the main result of this paper. Our strategy is to bound $\tau_{2b}$ by bounding $\tau_1$, $\tau_{2a}-\tau_1$, and $\tau_{2b}-\tau_{2a}$ separately.

In Section \ref{sec:uniform_control_of_r_k}, we bound $\tau_1$, and also establish uniform control over $\abs{r_k^2-1}$, which is needed for the rest of the proof. In Sections \ref{sec:phase_2a} and \ref{sec:phase2a_pt2}, we bound $\tau_{2a}-\tau_1$. This is the most difficult part of the proof, and relies on developing a nuanced notion of stochastic dominance with which we compare the sequence $\braces{s_k}_k$ with a carefully constructed sequence $\braces{\hat{s}_k}_k$. We then apply a small ball probability argument to the latter. In Section \ref{sec:phase_2b}, we bound $\tau_{2b} - \tau_{2a}$ by approximating the sequence with that obtained by removing the fluctuations. Finally, we will complete the proof of Theorem \ref{thm:main_theorem} in Section \ref{sec:proof_of_main_thm}, before concluding and discussing the broader implications of the result in Section \ref{sec:conclusion}.

\section{Uniform control over $r_k$} \label{sec:uniform_control_of_r_k}

In this section, we will show that the sequence of squared norms, $\braces{r_k^2}_k$, quickly converges to a small interval of width $O(\log d/\sqrt{d})$ around the value 1, and thereafter remains within this interval for at least $C d \log d$ iterations. In subsequent sections, we will show that this is sufficient time for the $\braces{s_k}_k$ sequence also to converge. The reason why such uniform control is necessary is because the formula for the horizontal update \eqref{eq:formula_for_beta} taken from a point $\boldy$ depends on $r(\boldy)$. By showing that $\braces{r_k^2}_k$ concentrates uniformly, we thereby also obtain control over the drift and fluctuations of $\braces{s_k}_k$, which enables us to do an essentially univariate analysis of this latter sequence.

\begin{lemma}[Bound for duration of Phase 1] \label{lem:bound_for_duration_of_phase1}
	There exists universal constants $C_1$ and $C_2$ such that $\tau_1 \leq C_2d\cdot(\log d + \log\abs{r_0^2-1})$ with probability at least $1-C_2/\log d$.
\end{lemma}

\begin{proof}
We first compute $\E\braces{r_k^2-1}$. Using \eqref{eq:state_space_update} and \eqref{eq:expectation_of_alpha}, we have

\begin{align*}
\E\braces{r_{k+1}^2-1~\vline~ \mathcal{F}_k} & = r_k^2-1 + \frac{1}{d}\E\braces{\alpha_{k+1}(\boldy^{(k)})~\vline~ \mathcal{F}_k} \\
& = \paren*{1-\frac{1}{d}}\paren*{r_k^2-1}.
\end{align*}

Iterating this gives

\begin{equation} \label{eq:contraction_in_exp_for_r_k}
\E\braces{r_k^2-1} = \paren*{1-\frac{1}{d}}^k\paren*{r_0^2-1}.
\end{equation}
Hence, whenever
\begin{equation} \label{eq:num_iterates_for_norm_convergence}
k \gtrsim \frac{\log d +\log(\abs{r_0^2-1})}{\log(1-1/d)} \approx d \paren*{\log d + \log\abs{r_0^2-1}}
\end{equation}
we have $\abs{\E\braces{r_k^2-1}} \leq \frac{1}{\sqrt{d}}$.

Next, we may obtain a recursive bound for the variance of the iterates using the law of total variance. We have

\begin{align} \label{eq:recursive_bound_for_r_k_var}
\Var\braces*{r_{k+1}^2-1} & = \E\braces{\Var\braces{ r_{k+1}^2-1 ~\vline~ \mathcal{F}_k}} + \Var\braces*{\E\braces{r_{k+1}^2-1 ~\vline~ \mathcal{F}_k}} \nonumber \\
& = \E\braces{\Var\braces{\alpha_{k+1}(\boldy^{(k)}) ~\vline~ \mathcal{F}_k}} + \Var\braces*{\paren*{1-1/d}\paren*{r_k^2-1}} \nonumber \\
& \leq \frac{C\E\braces{r_k^4}}{d^2} + \paren*{1-\frac{1}{d}}^2\Var\braces*{r_k^2-1}.
\end{align}

We may expand
\begin{equation*}
\E\braces{r_k^4} = \Var\braces{r_k^2 - 1} + \E\braces{r_k^2}^2,
\end{equation*}
and apply \eqref{eq:contraction_in_exp_for_r_k} to get
\begin{equation*}
\E\braces{r_k^2}^2 = \paren*{\E\braces{r_k^2-1}+1}^2 \lesssim \paren*{1-\frac{1}{d}}^k (r_0^2-1)^2 + 1.
\end{equation*}
Plugging these back into \eqref{eq:recursive_bound_for_r_k_var}, we get
\begin{equation*}
\Var\braces*{r_{k+1}^2-1} \leq \paren*{1-\frac{1}{2d}}\Var\braces*{r_k^2-1} + \frac{C}{d^2}\paren*{1-\frac{1}{d}}^k (r_0^2-1)^2 + \frac{C}{d^2},
\end{equation*}
and solving this recursion gives
\begin{equation*}
\Var\braces*{r_k^2-1} \leq \frac{C}{d^2} \cdot \paren*{k\cdot\paren*{1-1/2d}^k(r_0^2-1)^2 + \frac{1 - (1-1/2d)^{2k+2}}{1-(1-1/2d)^2}}.
\end{equation*}

It easy to check that this quantity is bounded by $C/d$ whenever \eqref{eq:num_iterates_for_norm_convergence} holds. In this case, we may apply Chebyshev's inequality to conclude that

\begin{align*}
r_k^2-1 \leq \E\braces{r_k^2-1} + \frac{\Var\braces{r_k^2-1}^{1/2}}{\delta} \leq \frac{1}{\sqrt{d}} + \frac{C}{\delta\sqrt{d}}
\end{align*}
with probability at least $1-\delta$. We may also similarly bound $r_k^2-1$ from below. Choosing $\delta = C/\log d$ gives us the probability bound we want.
\end{proof}

After $r_k$ has contracted to a value close to 1, we need to show that it remains close to 1 throughout the time scale needed for the algorithm to converge. Although it is clear from the formula \eqref{eq:formula_for_alpha} that the increments are subexponential, a naive union bound is not tight enough for our purposes. To overcome this, we make use of a maximal Bernstein inequality.

\begin{lemma}[Maximal Bernstein for subexponential martingale differences] \label{lem:uniform_Bernstein}
	Let $X_1,X_2,\ldots,X_M$ be a martingale difference sequence that is adapted to a filtration $\braces{\mathcal{G}_t}$. Suppose there is a constant $K > 0$ such that the following pointwise inequality holds almost surely for any time $t$ and $0 < \lambda \leq 1/2K$:
	\begin{equation*}
	\E\braces{e^{\lambda X_{t+1}} ~\vline~ \mathcal{G}_t} \leq e^{\lambda^2K^2}.
	\end{equation*}
	Denote $S_t = \sum_{i=1}^t X_i$, for $t = 1,2\ldots$. Then for all $\epsilon > 0$, we have the uniform tail bound
	\begin{equation} \label{eq:uniform_Bernstein}
	\P\braces{\exists t \leq M ~\colon~\abs{S_t} \geq \epsilon } \leq 2\exp\paren*{-\frac{1}{4}\braces*{\frac{\epsilon^2}{MK^2}\wedge\frac{\epsilon}{K}}}.
	\end{equation}
\end{lemma}

\begin{proof}
	This result is an easy consequence of combining two classical arguments: the supermartingale inequality and an exponential martingale inequality. This argument also appears with more sophistication in the exponential line-crossing method recently developed by \cite{Howard2018}. The proof details are deferred to Appendix \ref{sec:uniform_Bernstein}.
\end{proof}

\begin{remark}
	The tail bound on the right hand side of \eqref{eq:uniform_Bernstein} is exactly the Bernstein tail at time $M$. See \cite{Vershynin}.
\end{remark}

\begin{lemma}[Maximal Bernstein for process with contracting drift] \label{lem:contracting_markov_process}
	Let $W_1,W_2,\ldots,W_M$ be a real-valued stochastic process adapted to a filtration $\braces{\mathcal{G}_t}$. Suppose that there is some $0 < \rho < 1$ such that for each time $t$, we have \begin{equation} \label{eq:conditional_expectation_assumption}
	\E\braces{W_{t+1}~\vline~\mathcal{G}_t} = \rho W_t.
	\end{equation}
	Furthermore, assume that there is a constant $K > 0$ such that for any $t$, the following pointwise inequality holds almost surely  for any time $t$ and $0 < \lambda \leq 1/2K$:
	\begin{equation} \label{eq:subexponential_assumption}
	\E\braces{\exp\paren{\lambda (W_{t+1} - \rho W_t)} ~\vline~ \mathcal{G}_t} \leq e^{\lambda^2K^2}.
	\end{equation}
	Then for all $\epsilon > 0$, we have the uniform tail bound
	\begin{equation} \label{eq:uniform_bound_for_contracting_stoch_process}
	\P\braces{\exists t \leq M ~\colon~\abs{W_t} \geq \abs{W_0} + \epsilon } \leq 2\exp\paren*{-\frac{1}{4}\braces*{\frac{\epsilon^2}{MK^2}\wedge\frac{\epsilon}{K}}}.
	\end{equation}
\end{lemma}

\begin{proof}
	For any $t$, we set $X_t = W_t - \E\braces{W_t~\vline~\mathcal{G}_t}$ to obtain the recursive formula:
	\begin{equation*}
	W_{t+1} = \rho W_t + X_{t+1}.
	\end{equation*}
	Since $X_1,X_2,\ldots$ form a martingale difference sequence satisfying the assumptions of Lemma \ref{lem:uniform_Bernstein}, its partial sums are bounded, and we may apply an easy combinatorial result (Lemma \ref{lem:shrinking_partial_sums}) to bound the tails of $W_t$.
\end{proof}

We are now ready to state the guarantee that the radius remains uniformly close to 1.

\begin{lemma}[Uniform bounds for $r_k$] \label{lem:uniform_radius_bound}
	For any constant $C_1$, $M \leq C_1d\log d$, we have
	\begin{equation} \label{eq:radius_uniform_concentration}
	\sup_{0 \leq k \leq M} \abs{r_{\tau_1+k}^2-1} \leq \frac{C_2\log d}{\sqrt{d}}
	\end{equation}
	with probability at least $1-\frac{1}{d}$, where $C_2$ is a universal constant depending only on $C_1$. Furthermore, for $M \leq d^2/\log d$, there is a constant $C_3$ such that
	\begin{equation} \label{eq:radius_uniform_concentration_pt2}
	\sup_{0 \leq k \leq M} r_{\tau_1+k} \leq C_3
	\end{equation}
	with probability at least $1-1/d$.
\end{lemma}

\begin{proof}
	For ease of notation, assume $\tau_1 = 0$. Set $W_k \coloneqq r_k^2-1$ for $k=0,1,2,\ldots$. Recall also that $\braces{\mathcal{F}_k}_k$ is the filtration generated by the stochastic updates. We want to show that these satisfy the assumptions of Lemma \ref{lem:contracting_markov_process}. By Lemma \ref{lem:SGD_drift}, we see that \eqref{eq:conditional_expectation_assumption} is satisfied with $\rho = 1-\frac{1}{d}$, and we would like to use Lemma \ref{lem:psi_1_norm_bounds} to verify \eqref{eq:subexponential_assumption}. However, Lemma \ref{lem:psi_1_norm_bounds} is only valid when $\boldy$ lies in a bounded subset of $\mathcal{Y}$, and we don't assume this a priori.
	
	In order to overcome this, we use a coupling argument. Define an auxiliary seqeuence $\braces*{\tilde{r}_k}_k$ with the same initialization $\tilde{r}_0 = r_0$, and coupled with $r_k$ until the first time $\tau$ such that $\abs{r_\tau-1} \geq 1/2$. Thereafter, we evolve $\tilde{r}_k$ deterministically via the formula
	\begin{equation*}
	\tilde{r}_{k+1}^2 - 1 \coloneqq \paren*{1-\frac{1}{d}}\paren*{\tilde{r}_k^2-1}.
	\end{equation*}
	In other words, we enforce \eqref{eq:conditional_expectation_assumption} and \eqref{eq:subexponential_assumption} by fiat for the sequence $\braces{\tilde{r}_k^2-1}_{k \geq 0}$. This allows us to invoke Lemma \ref{lem:contracting_markov_process} for this sequence with $K = \frac{c}{d}$ for some universal constant $c > 0$. Plugging our choice of $M$ and $K$ into \eqref{eq:uniform_bound_for_contracting_stoch_process}, and choosing $\epsilon = \frac{C\log d}{\sqrt{d}}$, we that for $C$ large enough, the right hand side in \eqref{eq:uniform_bound_for_contracting_stoch_process} is bounded above by $\frac{1}{d}$.
	
	As such we have proved that
	\begin{equation*}
	\sup_{0 \leq k \leq M} \abs{\tilde{r}_k^2-1} \leq \frac{C\log d}{\sqrt{d}}
	\end{equation*}
	with probability at least $1-\frac{1}{d}$. Meanwhile, on the same event, we have $\tau > M$, so that the bound also holds for the original sequence, thereby giving us \eqref{eq:radius_uniform_concentration}. The proof of \eqref{eq:radius_uniform_concentration_pt2} is similar and is hence omitted.
\end{proof}

Let us denote $\tau_r \coloneqq \min\braces{k \colon \abs{r_{\tau_1+k}^2-1} \geq C_2\log d /\sqrt{d} }$. The previous lemma tells us that $\tau_r \geq C_1d\log d$ with probability at least $1-1/d$. In order to obtain the control over $\braces{s_k}_k$ promised at the start of this section, we would ideally like to condition on this event, but doing so would destroy the Markov nature of the process $\boldy^{(k)}$ and invalidate the update formula \eqref{eq:formula_for_beta}.

Instead, we will define a sequence $\braces{\tilde{\boldy}^{(k)}}_k$ that we evolve according to
\begin{equation}
\tilde{\boldy}^{(k+1)} = \begin{cases}
\boldy^{(k)} & \textnormal{if } k \leq \tau_1, \\
\Pi\paren*{\tilde{\boldy}^{(k)} + \frac{1}{d}\paren*{\alpha(\tilde{\boldy}^{(k)}),\beta(\tilde{\boldy}^{(k)})}} & \textnormal{if } k > \tau_1.
\end{cases}
\end{equation}
Here, $\Pi \colon \mathcal{Y} \to \mathcal{Y}$ leaves the $s$ component fixed, but projects the $r^2$ component to the interval $\sqbracket{1-C_2\log d/\sqrt{d}, 1 + C_2\log d/\sqrt{d}}$. We also continue to couple $\tilde{\boldy}^{(k)}$ and $\boldy^{(k)}$ for $k  > \tau_1$ by using the same random seeds for their updates.

To summarize, we have $\boldy^{(k)} = \tilde{\boldy}^{(k)}$ for $k \leq \tau_r$. Furthermore, the projection step means that the regularity of the horizontal updates for $\tilde{\boldy}^{(k)}$ is enforced by fiat. For the rest of this paper, we will work with this sequence instead, even going so far as to drop the tildes and use the same notation for both sequences, since the reason for the creation of this alternate sequence is purely technical and has no intuitive value.

\section{Phase 2a: Stochastic dominance argument} \label{sec:phase_2a}

We have broken up the analysis of ``Phase 2'' of the SGD process into two sub-phases. The first sub-phase concerns the portion of the process in which $\abs{s_k} = o(1)$, so that the drift and fluctuations are of comparable magnitudes. The goal of this section and the next is to bound the length of this phase by proving the following theorem.

\begin{theorem}[Bound for duration of Phase 2a] \label{thm:bound_tau_2a}
	There is some $0 < \gamma_1 < 1/2$ and success probability $p > 0$, such if we use this value of  $\gamma_1$ in \eqref{eq:tau_2a}, for $d$ large enough, $\tau_{2a} - \tau_1 \leq Cd\log d$ with probability at least $p$.
\end{theorem}

For ease of notation and making use of the strong Markov property, we may assume again that $\tau_1 = 0$. Our strategy for proving the theorem is to show that for iterations satisfying $\tau_1 \leq k \leq \tau_{2a}$, a subsequence of the horizontal marginals of the process, $\braces{s_{Bk}}_k$, stochastically dominate another process $\braces{\hat{s}_k}_k$ in magnitude. The second process will be constructed so that we will have precise control over its second and fourth moments, which will allow us to apply the Paley-Zygmund inequality to get a small ball probability argument. In this section, we will focus on constructing the comparison process and fleshing out the stochastic dominance argument.

In order to get control over moments, the process $\braces{\hat{s}_k}_k$ will be constructed to have approximately Gaussian increments. The comparison will thus be based on normal approximation using the Berry-Essen theory. It may be possible to obtain appropriate Berry-Essen bounds for martingale sequences, but the theory on this topic seems to be incomplete at the time of writing. We shall bypass this by showing that over epochs of an appropriate length $B$, the update $s_k \to s_{k+B}$ is well-approximated by the update we would get had we done a batch update, summing up $B$ independent steps taken from $\boldy^{(k)}$. Since this is a sum of i.i.d. random variables, the classical Berry-Eseen bound can then be applied to this latter update.

More precisely, Let $P$ denote the random mapping defined by \eqref{eq:SGD_update}. It will be helpful to use this notation in this section, as we will have to deal with a few different stochastic sequences. Fix the epoch length to be $B = d^{2/3}\log d$. We define a Markov kernel on the state space via the random mapping $Q$:
\begin{equation}
Q(\boldy) \coloneqq \boldy + \sum_{k=1}^B \paren*{\alpha_k(\boldy),\beta_k(\boldy)}
\end{equation}
where $\alpha_k(\boldy)$ and $\beta_k(\boldy)$ for $k=1,2,\ldots,B$ are independent realizations of the random variables defined in \eqref{eq:formula_for_alpha} and \eqref{eq:formula_for_beta}. Denoting $\sigma(\boldy)^2 \coloneqq \Var\braces{\beta(\boldy)}$, we note here that the drift and variance of this batch update satisfy the formulas 
\begin{equation*}
\E\braces{s(Q(\boldy))} = s(\boldy) + B\bar{\beta}(\boldy), \quad\quad\quad \Var\braces{s(Q(\boldy))} = B\sigma(\boldy)^2.
\end{equation*}

\begin{lemma}[Distance moved in one epoch] \label{lem:bound_for_distance_moved}
	Fix $\boldy^{(m)}$ for $\tau_1 \leq m \leq \tau_1 + \tau_r$. With probability at least $1-1/d^2$, we have
	\begin{equation*}
	\abs*{s_{m+t} - s_m} \leq \frac{CB}{d}\paren*{\abs*{s_m} + \sqrt{\frac{\log d}{B}}}
	\end{equation*}
	for $0 \leq t \leq B$.
\end{lemma}

\begin{proof}
	For ease of notation, re-index so that $m = 0$. First, observe that we have the decomposition
	\begin{align} \label{eq:martingale_decomposition_for_s}
	s_t - s_0 & = \sum_{k=0}^{t-1} s_{k+1} - s_{k} \nonumber \\
	& = \frac{1}{d}\sum_{k=0}^{t-1} \beta(\boldy^{(k)}) \nonumber \\
	& = \frac{1}{d}\sum_{k=0}^{t-1} \bar{\beta}(\boldy^{(k)}) + \frac{1}{d}\sum_{k=0}^{t-1} \paren*{\beta_k(\boldy^{(k)})- \bar{\beta}(\boldy^{(k)})}.
	\end{align}
	
	Applying the maximal Bernstein bound for martingale sequences (Lemma \ref{lem:uniform_Bernstein}) with $M = B$ and $\epsilon = \sqrt{B\log d}$, we get
	\begin{equation} \label{eq:phase_2a_uniform_bound_on_beta}
	\sup_{0 \leq t \leq B} \abs*{\sum_{k=0}^{t-1} \paren*{\beta_k(\boldy^{(k)})- \bar{\beta}(\boldy^{(k)})}} \leq C\sqrt{B\log d}.
	\end{equation}
	with probability at least $1-1/d^2$.
	
	By Lemma \ref{lem:lower_bound_for_drift}, we have
	\begin{equation*}
	\abs{\bar{\beta}(\boldy^{(t)})} \leq \bar{b}_{max}\abs{s_t}.
	\end{equation*}
	Plugging \eqref{eq:martingale_decomposition_for_s} into the above equation, and using the bound \eqref{eq:phase_2a_uniform_bound_on_beta}, we get
	\begin{align*}
	\abs{\bar{\beta}(\boldy^{(t)})} & \leq \bar{b}_{max} \paren*{\abs{s_0} + \frac{1}{d}\sum_{k=0}^{t-1} \abs*{\bar{\beta}(\boldy^{(k)})} + \frac{1}{d}\abs*{\sum_{k=0}^{t-1} \paren*{\beta_k(\boldy^{(k)})- \bar{\beta}(\boldy^{(k)})}}  } \\
	& \leq \bar{b}_{max} \paren*{\abs{s_0} + \frac{1}{d}\sum_{k=0}^{t-1} \abs*{\bar{\beta}(\boldy^{(k)})} + \frac{C\sqrt{B\log d}}{d}  }.
	\end{align*}
	We can simplify this recursive bound using some combinatorics. Applying Lemma \ref{lem:recursive_inequality} with $\rho = \bar{b}_{max}/d$, $x_t = \abs{\bar{\beta}(\boldy^{(t)})}$, and $\xi = \abs{s_0} + (C\sqrt{B\log d})/d$, we get
	\begin{align*}
	\frac{1}{d}\sum_{t=0}^{B-1} \abs{\bar{\beta}(\boldy^{(t)})} & \leq \paren*{\abs{s_0} + \frac{C\sqrt{B\log d}}{d}} \paren*{\paren*{1+\frac{\bar{b}_{max}}{d}}^B - 1} \\
	& \leq \frac{1.1\bar{b}_{max}B}{d}\paren*{\abs{s_0} + \frac{C\sqrt{B\log d}}{d}}.
	\end{align*}
	Plugging this back into \eqref{eq:martingale_decomposition_for_s}, for any $t \leq B$, we get
	\begin{align*}
	\abs{s_t-s_0} & \leq \frac{1.1\bar{b}_{max}B}{d}\paren*{\abs{s_0} + \frac{C\sqrt{B\log d}}{d}} + \frac{C\sqrt{B\log d}}{d} \\
	& \leq \frac{CB}{d} \paren*{ \abs{s_0} + \sqrt{\frac{\log d}{B}} }
	\end{align*}
	as we wanted.
\end{proof}

\begin{lemma}[Approximation error per epoch for $Q$ kernel] \label{lem:error_per_epoch}
	Fix $\boldy^{(m)}$ for $\tau_1 \leq m \leq \tau_1 + \tau_r$. For $d$ large enough, with probability at least $1-1/d^2$, there is a coupling such that we have
	\begin{equation*}
	\abs*{s\paren*{Q(\boldy^{(m)})} - s\paren*{P^B(\boldy^{(m)})}} \leq \frac{CB^2\log d}{d^2} \paren*{ \abs{s_m} + \sqrt{\frac{\log d}{B}} }.
	\end{equation*}
\end{lemma}

\begin{proof}
	For ease of notation, we again re-index so that $m=0$. Couple the random mappings $Q$ and $P^B$ by using the same draws for $u^{(t)}, v^{(t)}$, $t=1,2\ldots$. First condition on the probability $1-1/d^2$ event promised by Lemma \ref{lem:bound_for_distance_moved}. Now write
	\begin{align*}
	s\paren*{Q(\boldy^{(0)})} - s\paren*{P^B(\boldy^{(0)})} & = \frac{1}{d}\sum_{t=1}^B \paren*{\beta_t(\boldy) - \beta_t(\boldy^{(t-1)})} \\
	& = \frac{1}{d}\sum_{t=1}^B \paren*{r_t\cos\theta_t - r_0\cos\theta_0}(u^{(t)})^2 + \frac{1}{d}\sum_{t=1}^B \paren*{r_t\sin\theta_t - r_0\sin\theta_0} u^{(t)}v^{(t)} \\
	& \quad\quad\quad - \frac{2}{d} \sum_{t=1}^B \textbf{1}_{\mathcal{A}_t}\cdot (u^{(t)})^2,
	\end{align*}
	where for each $t$, $\mathcal{A}_t$ is the event that
	\begin{equation*}
	\sign(\cos\theta_t u^{(t)} + \sin\theta_t v^{(t)}) \neq \sign(\cos\theta_0 u^{(t)} + \sin\theta_0 v^{(t)}).
	\end{equation*}
	We further condition on the probability $1-1/d^2$ events in which we have
	\begin{equation*}
	\sup_{1 \leq t \leq B} \abs{u^{(t)}}, \sup_{1 \leq t \leq B} \abs{v^{(t)}} \leq C\sqrt{\log d}.
	\end{equation*}
	By Lemma \ref{lem:bound_for_distance_moved}, we see that the first term is bounded as follows:
	\begin{align*}
	\abs*{\frac{1}{d}\sum_{t=1}^B \paren*{r_t\cos\theta_t - r_0\cos\theta_0}(u^{(t)})^2} & \leq \frac{B}{d}\sup_{0 \leq t \leq B-1} \abs{s_t-s_0} \cdot \sup_{1 \leq t \leq B} (u^{(t)})^2 \\
	& \leq  \frac{CB^2 \log d}{d^2}\paren*{\abs*{s_0} + \sqrt{\frac{\log d}{B}}}.
	\end{align*}
	The second term may be bounded similarly. To bound the third term, we use Lemma \ref{lem:Chernoff_for_martingale_difference_seq}. Observe that $\braces{\textbf{1}_{\mathcal{A}_t}}$ is a sequence of Bernoulli variables. If we condition on the high probability event promised by Lemma \ref{lem:bound_for_distance_moved}, we also have
	\begin{equation} \label{eq:bound_for_event_prob}
	\E\braces{\textbf{1}_{A_t} ~|~\mathcal{F}_{k-1}} \leq C\abs{\theta_t-\theta_0} \leq \frac{CB}{d}\paren*{\abs*{s_0} + \sqrt{\frac{\log d}{B}}}.
	\end{equation}
	In Lemma \ref{lem:Chernoff_for_martingale_difference_seq}, set $M=B$, $\epsilon = 1/2$, and $\theta$ to be the right hand side of \eqref{eq:bound_for_event_prob}. This gives
	\begin{equation*}
	\sum_{t=1}^B 1_{\mathcal{A}_t} \leq \frac{CB^2}{d}\paren*{\abs*{s_0} + \sqrt{\frac{\log d}{B}}}
	\end{equation*}
	with probability at least
	\begin{equation*}
	1-\exp\paren*{-C B^2/d \cdot \sqrt{\log d/B}} \geq 1 - \exp\paren*{- C\log^2 d} \geq 1 - 1/d^2
	\end{equation*}
	for $d$ large enough. On this event, the third term is therefore also bounded by $\frac{CB^2 \log d}{d^2}\paren*{\abs*{s_0} + \sqrt{\frac{\log d}{B}}}$.
	
	By adjusting constants if necessary, we can make sure that the total error probability is bounded by $1/d^2$.
\end{proof}

In order to compare $\braces{s_{Bk}}_k$ with our (as yet undefined) reference sequence $\braces{\hat{s}_k}_k$, we will use a more nuanced version of the usual notion of stochastic dominance.

\begin{definition}[Stochastic dominance with error]
	Given two real-valued random variables $X$ and $Y$ defined on the same probability space, we say that \emph{$X$ dominates $Y$ in place up to error $\delta$}, denoted $X \geq_{\delta} Y$, if $X \geq Y$ with probability at least $1-\delta$. In addition, given any two real-valued random variables $X$ and $Y$, we say that \emph{$X$ stochastically dominates $Y$ up to error $\delta$}, if there is a coupling of $X$ and $Y$ for which $X \geq_{\delta} Y$. Denote this relation by $X \succeq_{\delta} Y$. Note that the usual notion of stochastic dominance is equivalent to $X \succeq_0 Y$, which we will also denote using $X \succeq Y$.
\end{definition}

We first establish stochastic dominance for the conditional distribution of $\abs{s_{k+B}}$ given $\boldy^{(k)}$ over that obtained from an approximately Gaussian increment. Following this, we will argue that stochastic dominance of individual steps also implies stochastic dominance for the entire process.

\begin{lemma}[One step stochastic dominance] \label{lem:one_step_stochastic_dominance}
	Define the region
	\begin{equation} \label{eq:defn_of_D}
	\mathcal{D} \coloneqq \braces*{(r^2,s) \in \mathcal{Y} ~\colon~ \abs{r^2-1} \leq \frac{C\log d}{\sqrt{d}}, \abs{s} < \gamma_1 },           
	\end{equation}
	and let $\boldy$ be any point in $\mathcal{D}$. Define the approximation error term
	\begin{equation} \label{eq:defn_of_epsilon(s)}
	\epsilon(s) \coloneqq \frac{CB^2\log d}{d^2} \paren*{\abs{s} \vee \sqrt{\frac{\log d}{B}}}.
	\end{equation}
	For any $a > 0$, we also denote the soft-thresholding operator $\rho_a \colon \R \to \R$ via
	\begin{equation*}
	\rho_a\sqbracket{x} = \sign(x)\paren*{\abs{x}-a}_+.
	\end{equation*}
	Then with $\delta = 1/d^2 + C/\sqrt{B}$, we have
	\begin{equation} \label{eq:stochastic_dominance_of_P_over_g}
	s(P^B(\boldy))^2 \succeq_{\delta} \rho_{\epsilon(s(\boldy))}\sqbracket*{s(\boldy) + B\bar{\beta}(\boldy) + \sqrt{B}\sigma(\boldy)g}^2.
	\end{equation}
\end{lemma}

\begin{proof}
Fix $\boldy$, and set $W_B \coloneqq \frac{1}{\sigma(\boldy)\sqrt{B}}\sum_{k=1}^B \paren*{\beta_k(\boldy) - \bar{\beta}(\boldy)}$, where $\sigma(\boldy)^2 \coloneqq \Var\braces{\beta(\boldy)}$. The subexponential bound on $\beta(\boldy)$ implies bounds on the 3rd moments, so by Berry-Esseen, its distribution function $F_W$ satisfies
\begin{equation*}
\norm{F_W - \Phi}_\infty \leq \frac{C}{\sqrt{B}}
\end{equation*}
where $\Phi$ is the distribution function of a standard normal random variable, and $C$ is an absolute constant.

Setting $\delta = C/\sqrt{B}$, we have have $W_B \succeq_{\delta} g$ by Lemma \ref{cor:CDF_distance_and_coupling}. Since
\begin{equation*}
s(Q(\boldy)) = s(\boldy) + B\bar{\beta}(\boldy) + \sqrt{B}\sigma(\boldy)\cdot W_B,
\end{equation*}
it is easy to see (Lemma \ref{lem:monotone_transformations}) that
\begin{equation*}
s(Q(\boldy)) \succeq_\delta s(\boldy) + B\bar{\beta}(\boldy) + \sqrt{B}\sigma(\boldy)\cdot g.
\end{equation*}
We next apply Lemma \ref{lem:kolmogorov_dist_properties} and \ref{lem:monotone_transformations} to get
\begin{equation} \label{eq:stochastic_dominance_of_Q_over_g}
\rho_{\epsilon(s(\boldy))}\sqbracket*{s(Q(\boldy))}^2 \succeq_\delta \rho_{\epsilon(s(\boldy))}\sqbracket*{ s(\boldy) + B\bar{\beta}(\boldy) + \sqrt{B}\sigma(\boldy)\cdot g}^2.
\end{equation}

By Lemma \ref{lem:error_per_epoch}, we also have
\begin{equation*}
s(P^B(\boldy))^2 \succeq_{\delta'} \rho_{\epsilon(s(\boldy))}\sqbracket*{s(Q(\boldy))}^2
\end{equation*}
for $\delta' = 1/d^2$. We may combine this with \eqref{eq:stochastic_dominance_of_Q_over_g} using the transitivity of stochastic dominance (Lemma \ref{lem:transitivity_of_stochastic_dominance}), to get the bound we want in \eqref{eq:stochastic_dominance_of_P_over_g}.
\end{proof}

We are now ready to formally define the process $\braces{\hat{s}_k}_k$. Let $L\colon \R\cross\mathcal{Y}\to \R$ be a transition kernel defined by
\begin{equation*}
L(s,\boldy) \coloneqq \rho_{\epsilon(s)}\sqbracket*{b(s) + \sqrt{B}\sigma(\boldy)g},
\end{equation*}
where
\begin{equation} \label{eq:defn_of_b(s)}
b(s) \coloneqq (1+\kappa B/d)s,
\end{equation}
and $0 < \kappa < 1$ a small constant to be determined later, while $\sigma(\boldy)^2 \coloneqq \Var\braces{\beta(\boldy)}$ as before. Again setting $\tau_1 = 0$ for notational convenience, we define
\begin{equation} \label{eq:defn_of_hat_s}
\hat{s}_{k+1} \coloneqq L(\hat{s}_k,\boldy^{(kB\wedge \tau_{2a})}).
\end{equation}

\begin{lemma}[Process-wise stochastic dominance] \label{lem:process_wise_stochastic_dominance}
	Let $\hat{s}_0,\hat{s}_1,\hat{s}_2,\ldots$ be the Markov process defined by the update rule \eqref{eq:defn_of_hat_s}. Then for any positive integer $k$, we have
	\begin{equation*}
	s_{kB \wedge \tau_{2a}}^2 \succeq_{k\delta} \hat{s}_{k\wedge \floor{\tau_{2a}/B}}^2,
	\end{equation*}
	where $\delta = 1/d^2 + C/\sqrt{B}$.
\end{lemma}

\begin{proof}
	For convenience of notation, we define the auxiliary transitional kernel
	\begin{equation*}
	K(s,\boldy) \coloneqq \rho_{\epsilon(s(\boldy))}\sqbracket*{s(\boldy) + B\bar{\beta}(\boldy) + \sqrt{B}\sigma(\boldy)g}.
	\end{equation*}
	Consider a fixed $\boldy \in \mathcal{D}$. In Lemma \ref{lem:one_step_stochastic_dominance}, we showed that $s(P^B(\boldy))^2 \succeq_{\delta} K(s(\boldy),\boldy)^2$. If we choose $\kappa$ small enough in \eqref{eq:defn_of_b(s)}, then by the drift lower bound in Lemma \ref{lem:lower_bound_for_drift}, replacing the kernel $K$ with $L$ simply reduces the magnitude of the drift while preserving the variance of the Gaussian increment and the magnitude of the soft-thresholding. Using the first part of Lemma \ref{lem:stochastic_dominance_for_truncated_Gaussians} thus gives the bound $K(s(\boldy),\boldy)^2 \succeq L(s(\boldy),\boldy)^2$. This implies through transitivity (Lemma \ref{lem:transitivity_of_stochastic_dominance}) that
	\begin{equation} \label{eq:stochastic_dominance_of_P_over_L}
	s(P^B(\boldy))^2 \succeq_{\delta} L(s(\boldy),\boldy)^2.
	\end{equation}
	
	Treating each epoch of updates as a single step, what we have showed is the stochastic dominance of a single step over a Gaussian increment, conditioned on a fixed starting value for $\boldy$. It is not so easy, however, to conclude that stochastic dominance is preserved when composing multiple steps together. This is because we are working with transition kernels, and not sums of independent random variables. To overcome this, we need to use the second part of Lemma \ref{lem:stochastic_dominance_for_truncated_Gaussians}, which tells us that
	\begin{equation} \label{eq:stochastic_dominance_for_L}
	L(s',\boldy)^2 \succeq L(s,\boldy)^2.
	\end{equation}
	whenever $s^2 \leq (s')^2$.
	
	Let us now prove the claim by induction. The case $k = 0$ is clear since $\hat{s}_0 = s_0$. Now assume that the statement holds for some $k$. Let $\mathcal{G}_k$ be the $\sigma$-algebra generated by $\boldy^{(0)},\ldots \boldy^{(k)}, \hat{s}_0^2,\ldots \hat{s}_k^2$. Condition on $\mathcal{G}_k$ as well as the $1-k\delta$ probability event for which
	\begin{equation*}
	s_{kB \wedge \tau_{2a}}^2 \geq \hat{s}_{k\wedge \floor{\tau_{2a}/B}}^2.
	\end{equation*}
	
	If $\tau_{2a} \leq kB$, then $s_{(k+1)B \wedge \tau_{2a}} = s_{kB \wedge \tau_{2a}}$ and $\hat{s}_{(k+1)\wedge \floor{\tau_{2a}/B}} = \hat{s}_{k\wedge \floor{\tau_{2a}/B}}$, so that
	\begin{equation} \label{eq:inductive_statement_to_be_proved}
	s_{(k+1)B \wedge \tau_{2a}}^2 \geq \hat{s}_{(k+1)\wedge \floor{\tau_{2a}/B}}^2
	\end{equation}
	on the same event. Otherwise, since $\boldy^{(kB)} \in \mathcal{D}$, we have
	\begin{equation*}
	\hat{s}_{k+1}^2 = L(\hat{s}_k,\boldy^{(kB)})^2  \preceq L(s_{kB},\boldy^{(kB)})^2 \preceq_{\delta} s\paren*{P^B(\boldy^{(kB)})}^2 = s_{(k+1)B}^2.
	\end{equation*}
	Here, the first dominance bound follows from \eqref{eq:stochastic_dominance_for_L}, while the second follows from \eqref{eq:stochastic_dominance_of_P_over_L}. This means that we can construct a coupling of the update kernels such that $\hat{s}_{k+1}^2 \leq s_{(k+1)B}^2$ with conditional probability at least $1-\delta$.
	
	If $\tau_{2a} \geq (k+1)B$, then this immediately implies that \eqref{eq:inductive_statement_to_be_proved} holds. On the other hand, if $kB < \tau_{2a} < (k+1)B$, we have $\floor{\tau_{2a}/B} = k$ and $s_{kB}^2 < \gamma_1^2$, so that
	\begin{equation*}
	\hat{s}_{(k+1)\wedge \floor{\tau_{2a}/B}}^2 = \hat{s}_k^2 \leq s_{kB}^2 < \gamma_1^2 \leq s_{\tau_{2a}}^2 = s_{(k+1)B \wedge \tau_{2a}}^2.
	\end{equation*}
	As such, the statement also holds for $k+1$.
\end{proof}

\section{Phase 2a: Small ball probability argument via Paley-Zygmund} \label{sec:phase2a_pt2}

Recall that our goal is to obtain a small ball probability bound for $s_k^2$, for $k \gtrsim d\log d$, thereby proving Theorem \ref{thm:bound_tau_2a}. By the stochastic dominance argument in the last section, it suffices to obtain such a bound for $\hat{s}_k^2$, with the appropriate time rescaling. This is the plan for this section, and we start by observing the following general recursive bounds for moments of adapted sequences.

\begin{lemma}[Recursive moment bounds for adapted sequences] \label{lem:martingale_orthogonality}
	Let $S_1,S_2,\ldots$ be a process adapted to the filtration $\braces{\mathcal{G}_t}$. Then we have
	\begin{equation} \label{eq:general_2nd_moment_recursive_bound}
	\E\braces{S_{t+1}^2} = \E\braces{ \E\braces{S_{t+1}~\vline~\mathcal{G}_t})^2} + \E\braces{(S_{t+1}-\E\braces{S_{t+1}~\vline~\mathcal{G}_t})^2}.
	\end{equation}
	
	\begin{equation} \label{eq:general_4th_moment_recursive_bound}
	\E\braces{S_{t+1}^4}^{1/2} \leq \E\braces{\E\braces{S_{t+1}~\vline~\mathcal{G}_t}^4}^{1/2} + 4\E\braces{(S_{t+1}-\E\braces{S_{t+1}~\vline~\mathcal{G}_t})^4}^{1/2}.
	\end{equation}
\end{lemma}

\begin{proof}
	The first equation is standard. To prove the second, we expand and use martingale orthogonality to write
	\begin{align} \label{eq:martingale_orthogonality}
	\E\braces{S_{t+1}^4} &= \E\braces{(S_{t+1}-\E\braces{S_{t+1}~\vline~\mathcal{G}_t})^4} + \E\braces{\E\braces{S_{t+1}~\vline~\mathcal{G}_t}^4} + 6\E\braces{(S_{t+1}-\E\braces{S_{t+1}~\vline~\mathcal{G}_t})^2\E\braces{S_{t+1}~\vline~\mathcal{G}_t}^2} \nonumber \\
	& \quad\quad\quad\quad + 4\E\braces{(S_{t+1}-\E\braces{S_{t+1}~\vline~\mathcal{G}_t})^3\E\braces{S_{t+1}~\vline~\mathcal{G}_t}}.
	\end{align}
	By Cauchy-Schwarz, we have
	\begin{equation*}
	\E\braces{(S_{t+1}-\E\braces{S_{t+1}~\vline~\mathcal{G}_t})^2\E\braces{S_{t+1}~\vline~\mathcal{G}_t}^2} \leq \E\braces{(S_{t+1}-\E\braces{S_{t+1}~\vline~\mathcal{G}_t})^4}^{1/2} \cdot \E\braces{\E\braces{S_{t+1}~\vline~\mathcal{G}_t}^4}^{1/2}.
	\end{equation*}
	Furthermore, the third term can be bounded as follows:
	\begin{align*}
	\E\braces{(S_{t+1}-\E\braces{S_{t+1}~\vline~\mathcal{G}_t})^3\E\braces{S_{t+1}~\vline~\mathcal{G}_t}} & =  \E\braces{(S_{t+1}-\E\braces{S_{t+1}~\vline~\mathcal{G}_t})^2 \cdot \paren*{(S_{t+1}-\E\braces{S_{t+1}~\vline~\mathcal{G}_t})\E\braces{S_{t+1}~\vline~\mathcal{G}_t}}} \\
	& \leq \frac{1}{2}\E\braces{(S_{t+1}-\E\braces{S_{t+1}~|~\mathcal{G}_t})^4} + \frac{1}{2}\E\braces*{(S_{t+1}-\E\braces{S_{t+1}~|~\mathcal{G}_t})^2\E\braces{S_{t+1}~|~\mathcal{G}_t}^2}.
	\end{align*}

	Plugging these into the original equation, we get
	\begin{equation*}
	\E\braces{S_{t+1}^4} \leq \paren*{\E\braces{\E\braces{S_{t+1}~\vline~\mathcal{G}_t})^4}^{1/2} + 4\E\braces{(S_{t+1}-\E\braces{X_{t+1}~\vline~\mathcal{G}_t})^4}^{1/2} }^2
	\end{equation*}
	as we wanted.
\end{proof}

In order to apply these bounds to our sequence $\braces{\hat{s}_k}_k$, we will need estimates of the moments of each increment.

\begin{lemma}[Moment bounds for $\hat{s}_k$ increments]
	We have the following:
	\begin{equation} \label{eq:lower_bound_for_2nd_moment}
	\E\braces*{\paren*{\hat{s}_{k+1} - \E\braces*{\hat{s}_{k+1}~\vline~\mathcal{F}_{kB}}}^2 ~\vline~\mathcal{F}_{kB}} \geq \frac{C_1B}{d^2}
	\end{equation}
	\begin{equation} \label{eq:upper_bound_for_4th_moment}
	\E\braces*{\paren*{\hat{s}_{k+1} - \E\braces*{\hat{s}_{k+1}~\vline~\mathcal{F}_{kB}}}^4 ~\vline~\mathcal{F}_{kB}}^{1/2} \leq \frac{C_2B}{d^2}
	\end{equation}
\end{lemma}

\begin{proof}
	In this proof we shall, for convenience, denote $X \coloneqq b(\hat{s}_k) + \sqrt{B}\sigma(\boldy^{(kB)})g$ and $\epsilon \coloneqq \epsilon(\hat{s}_k)$. Writing out the definition \eqref{eq:defn_of_hat_s}, we then have $\hat{s}_{k+1} = \rho_{\epsilon}\sqbracket*{X}$. Using the definition of $X$ and the fact that $\boldy^{(kB)} \in \mathcal{D}$, we compute
	\begin{equation} \label{eq:bound_for_var_X}
	\Var\braces*{X~\vline~\mathcal{F}_{kB}} = \Var\braces*{\sqrt{B}\sigma(\boldy^{(kB)})g} = B\sigma(\boldy^{(kB)})^2 \geq \frac{CB}{d^2}.
	\end{equation}
	We need to show that soft-thresholding $X$ does not decrease its variance by two much, and will consider two cases depending on the value of $\hat{s}_k$. We shall suppose WLOG that $\hat{s}_k > 0$.
	
	First, suppose $\hat{s}_k \leq \frac{1}{\log^3 d}$. Then
	\begin{align*}
	\Var\braces*{\rho_{\epsilon}\sqbracket*{X}~\vline~\mathcal{F}_{kB}}^{1/2} & \geq \E\braces*{\abs*{\rho_{\epsilon}\sqbracket*{X}- \E\rho_{\epsilon}\sqbracket*{X}} ~\vline~\mathcal{F}_{kB}} \\
	& \geq \E\braces*{\abs*{X-\E X} - \abs*{X - \rho_{\epsilon}\sqbracket{X}} - \abs*{\E X - \E\rho_{\epsilon}\sqbracket{X}}~\vline~ \mathcal{F}_{kB}} \\
	& \geq \E\braces*{\abs*{X-\E X} ~\vline~ \mathcal{F}_{kB}} - 2\epsilon \\
	& = \sqrt{\frac{2}{\pi}}\Var\braces*{X~\vline~\mathcal{F}_{kB}}^{1/2} - 2\epsilon.
	\end{align*}
	Plugging our assumption on $\hat{s}_k$ into \eqref{eq:defn_of_epsilon(s)}, we get
	\begin{equation*}
	\epsilon \leq \frac{B^2\log d}{d^2} \frac{1}{\log^3 d} = \frac{1}{d^{2/3}} \ll \frac{\log^{1/2} d}{d^{2/3}} = \frac{\sqrt{B}}{d}.
	\end{equation*}
	By \eqref{eq:bound_for_var_X}, the quantity on the right hand side is a lower bound for $\Var\braces*{X~\vline~\mathcal{F}_{kB}}$, giving us what we want.
	
	Now assume instead that $\hat{s}_k \geq \frac{1}{\log^3 d}$. Let $X'$ be an independent copy of $X$. Using a well-known formula for variance, we have
	\begin{align} \label{eq:truncated_var_working}
	\Var\braces*{\rho_{\epsilon}\sqbracket*{X}~\vline~\mathcal{F}_{kB}} & = \frac{1}{2}\E\braces*{\paren*{\rho_{\epsilon}\sqbracket*{X} - \rho_\epsilon\sqbracket{X'}}^2~\vline~\mathcal{F}_{kB}} \nonumber \\
	& \geq \frac{1}{2}\E\braces*{\paren*{\rho_{\epsilon}\sqbracket*{X} - \rho_\epsilon\sqbracket{X'}}^2 \cdot \mathbf{1}\braces{ X, X' \geq \epsilon}~\vline~\mathcal{F}_{kB}} \nonumber \\
	& = \frac{1}{2}\E\braces*{\paren*{X - X'}^2 \cdot \mathbf{1}\braces{ X, X' \geq \epsilon}~\vline~\mathcal{F}_{kB}} \nonumber \\
	& = \frac{B}{2}\sigma(\boldy^{(kB)})^2 \cdot \E\braces*{(g-g')^2 \cdot \mathbf{1}\braces*{g,g' \geq \frac{-b(\hat{s}_k) + \epsilon(\hat{s}_k)}{\sqrt{B}\sigma(\boldy^{(kB)})} }},
	\end{align}
	where $g$ and $g'$ are independent standard normal random variables.
	
	Observe that
	\begin{align*}
	\frac{-b(\hat{s}_k) + \epsilon(\hat{s}_k)}{\sqrt{B}\sigma(\boldy^{(kB)})} & = \frac{\paren*{-(1+\kappa B /d) + B^2\log d/d^2} \hat{s}_k}{\sqrt{B}\sigma(\boldy^{(kB)})} \\
	& \leq  - \frac{\hat{s}_k}{\sqrt{B}\sigma(\boldy^{(kB)})} \\
	& \lesssim  - \frac{d}{\sqrt{B}\log^3 d},
	\end{align*}
	which tends to $- \infty$, so the expectation on the right hand side of \eqref{eq:truncated_var_working} tends to $\E\braces{(g-g')^2} = 2$. As such, the entire quantity on the right hand side is bounded from below by $CB/d^2$.

	Next, to obtain \eqref{eq:upper_bound_for_4th_moment}, we simply apply Lemma \ref{lem:4th_moment_of_contraction} to get
	\begin{align*}
	\E\braces*{\paren*{\hat{s}_{k+1} - \E\braces*{\hat{s}_{k+1}~\vline~\mathcal{F}_{kB}}}^4 ~\vline~\mathcal{F}_{kB}} & \leq 8 \E\braces*{\paren*{b(\hat{s}_k) + \sqrt{B}\sigma(\boldy^{(kB)})g - \E\braces{b(\hat{s}_k) + \sqrt{B}\sigma(\boldy^{(kB)})~\vline~\mathcal{F}_{kB}}}^4\vline\mathcal{F}_{kB}} \\
	& = 8B^2\sigma(\boldy^{(kB)})^4\E\braces{g^4}.
	\end{align*}
	By our constraint that $\boldy^{(kB)} \in \mathcal{D}$, we have $\sigma(\boldy^{kB})^2 \leq \frac{C}{d^2}$, thereby giving the bound we want.
\end{proof}

\begin{lemma}[Recursive moment bounds for $\hat{s}_k$ process] \label{lem:recursive_moment_bounds}
	We have the recursive formulas
	\begin{equation} \label{eq:2nd_moment_recursive_bound}
	\E\braces{\hat{s}_{k+1}^2} \geq \paren*{1 + \frac{\kappa B}{d}\paren*{1-\frac{C_1}{\log^2 d}}}^2\E\braces{\hat{s}_k^2} + \frac{C_2B}{d^2}
	\end{equation}
	\begin{equation} \label{eq:4th_moment_recursive_bound}
	\E\braces{\hat{s}_{k+1}^4}^{1/2} \leq \paren*{1 + \frac{\kappa B}{d}}^2\E\braces{\hat{s}_k^4}^{1/2} + \frac{C_3B}{d^2}.
	\end{equation}
\end{lemma}

\begin{proof}
	The second inequality is easier to prove, so we shall start with this. We may use Lemma \ref{lem:4th_moment_of_contraction} to get
	\begin{align*}
	\abs*{\E\braces*{\hat{s}_{k+1}~\vline~\mathcal{F}_{kB}}} & = \abs*{\E\braces*{\rho_{\epsilon(\hat{s}_k)}\sqbracket{b(\hat{s}_k) + \sqrt{B}\sigma(\boldy^{(kB)})g} ~\vline~ \mathcal{F}_{kB}}} \\
	& \leq \abs*{b(\hat{s}_k)} \\
	& = \paren*{1+\frac{\kappa B}{d}}\abs*{\hat{s}_k}.
	\end{align*}
	Plugging this bound together with \eqref{eq:upper_bound_for_4th_moment} into \eqref{eq:general_4th_moment_recursive_bound} gives us \eqref{eq:4th_moment_recursive_bound}.
	
	Next, denoting $X \coloneqq b(\hat{s}_k) + \sqrt{B}\sigma(\boldy^{(kB)})g$ and $\epsilon \coloneqq \epsilon(\hat{s}_k)$ as in the previous lemma, observe that
	\begin{align} \label{eq:from_contraction_to_subtraction}
	\abs*{\E\braces*{\hat{s}_{k+1}~\vline~\mathcal{F}_{kB}}}& =\abs*{\E\braces*{\rho_{\epsilon}\sqbracket{X} ~\vline~ \mathcal{F}_{kB}}} \nonumber \\
	& \geq \paren*{\abs*{\E\braces*{X ~\vline~ \mathcal{F}_{kB}}} - \abs*{\E\braces*{\rho_{\epsilon}\sqbracket{X} - X ~\vline~ \mathcal{F}_{kB}}}}_+ \nonumber \\
	& \geq \paren*{\paren*{1+\frac{\kappa B}{d}}\abs*{\hat{s}_k} - \epsilon(\hat{s}_k)}_+.
	\end{align}
	
	It remains to compare the relative magnitudes of the two terms. For convenience, we reproduce the definition of $\epsilon(\hat{s}_k)$ here:
	\begin{equation*}
	\epsilon(\hat{s}_k) \coloneqq \frac{CB^2\log d}{d^2} \paren*{\abs{\hat{s}_k} \vee \sqrt{\frac{\log d}{B}}}.
	\end{equation*}
	We will divide the proof into three cases, depending on the magnitude of $\hat{s}_k$. WLOG, assume that $\hat{s}_k \geq 0$.
	
	When $\hat{s}_k \geq \frac{1}{d^{1/3}}$, then
	\begin{equation*}
	\epsilon(\hat{s}_k) = \frac{CB^2\log d}{d^2} \hat{s}_k \leq \frac{1}{\log^2 d}\cdot\frac{\kappa B}{d}\hat{s}_k
	\end{equation*}
	for $d$ large enough, and when $\frac{\log^4 d}{d^{2/3}} \leq \hat{s}_k \leq \frac{1}{d^{1/3}}$, one may check that
	\begin{equation*}
	\epsilon(\hat{s}_k) = \frac{CB^2\log d}{d^2} \sqrt{\frac{\log d}{B}} \leq \frac{C}{\log^2 d}
	\cdot\frac{\kappa B}{d}\hat{s}_k.
	\end{equation*}
	In either case, we may plug the bound for $\epsilon(\hat{s}_k)$ into \eqref{eq:from_contraction_to_subtraction} to get
	\begin{equation*}
	(\E\braces{\hat{s}_{k+1} ~|~ \mathcal{F}_{kB}})^2 \geq \paren*{1 + \frac{\kappa B}{d}\paren*{1-\frac{C}{\log^2 d}}}^2\hat{s}_k^2.
	\end{equation*}
	
	Next, when $0 \leq \hat{s}_k \leq \frac{\log^4 d}{d^{2/3}}$, we have
	\begin{equation*}
	\frac{CB^2\log d}{d^2} \paren*{\hat{s}_k \vee \sqrt{\frac{\log d}{B}}} \leq \frac{C\log^3 d}{d},
	\end{equation*}
	so that we have
	\begin{equation*}
	\paren*{\paren*{1 + \frac{\kappa B}{d}}\hat{s}_{k} - \frac{CB^2\log d}{d^2} \paren*{\hat{s}_k \vee \frac{\log d}{\sqrt{B}}}}_+^2 \geq \paren*{1 + \frac{\kappa B}{d}}^2\hat{s}_k^2 - \frac{C\log^3 d}{d} \hat{s}_k.
	\end{equation*}
	Furthermore, we can bound the second term on the right hand side via
	\begin{equation*}
	\frac{\log^3 d}{d} \hat{s}_k \leq \frac{\log^7 d}{d^{5/3}} \lesssim \frac{B}{d^2 \log d}.
	\end{equation*}
	As such, we may write
	\begin{align*}
	\E\braces{\hat{s}_{k+1} ~|~ \mathcal{F}_{kB}}^2 \geq \paren*{1 + \frac{\kappa B}{d}}^2\hat{s}_k^2 - \frac{cB}{d^2}\paren*{\frac{1}{\log d}}.
	\end{align*}
	
	We may thus take expectations to get
	\begin{align*}
	\E\braces{\E\braces{\hat{s}_{k+1} ~|~ \mathcal{F}_{kB}})^2} \geq \paren*{1 + \frac{\kappa B}{d}\paren*{1-\frac{1}{\log^2 d}}}^2\E\braces*{\hat{s}_k^2} - \frac{cB}{d^2}\paren*{\frac{1}{\log^2 d}}.
	\end{align*}
	Combining this with \eqref{eq:general_2nd_moment_recursive_bound} and \eqref{eq:lower_bound_for_2nd_moment} gives \eqref{eq:2nd_moment_recursive_bound}.
\end{proof}

\begin{proof}[Proof of Theorem \ref{thm:bound_tau_2a}]
	For convenience, denote $A = 1 - 1/\log^2 d$. We solve the recursion in \eqref{eq:2nd_moment_recursive_bound} to get
	\begin{equation} \label{eq:first_lower_bound_for_sk_hat}
	\E\braces{\hat{s}_k^2} \geq \frac{(1+\kappa AB/d)^{2k+2} - 1}{(1+\kappa AB / d)^2 - 1} \cdot \frac{C_2B}{d^2} + \paren*{1+\frac{\kappa AB}{d}}^{2k} \hat{s}_0^2.
	\end{equation}
	Note that 
	\begin{equation*}
	\frac{1}{(1+\kappa AB / d)^2 - 1} \cdot \frac{C_2 B}{d^2} \asymp \frac{d}{\kappa AB}\cdot \frac{B}{d^2} \asymp \frac{1}{\kappa d},
	\end{equation*}
	so that we can simplify \eqref{eq:first_lower_bound_for_sk_hat} to
	\begin{equation} \label{eq:lower_bound_for_sk_hat}
	\E\braces{\hat{s}_k^2} \gtrsim \paren*{\paren*{1+\frac{\kappa AB}{d}}^{2k+2} - 1}\frac{C}{\kappa d} + \paren*{1+\frac{\kappa AB}{d}}^{2k} \hat{s}_0^2.
	\end{equation}
	Meanwhile, for any $T \asymp \frac{Cd}{B}\log d \asymp \frac{\log d}{\log(1+\kappa AB/d)}$, we have
	\begin{equation} \label{eq:choice_of_T}
	\paren*{1+\frac{\kappa AB}{d}}^{2T+2} = \exp\paren*{C\log d}.
	\end{equation}
	Putting these together, we see that regardless of the value of $\hat{s}_0$, there is an absolute constant $\gamma$ such that $\E\braces{\hat{s}_T^2} \geq \gamma$ if we choose the constant $C$ in the definition of $T$ to be large enough.
	
	On the other hand, we may solve the second recursion \eqref{eq:4th_moment_recursive_bound} and apply a similar computation as before to get
	\begin{equation} \label{eq:upper_bound_for_sk_hat}
	\E\braces{\hat{s}_k^4}^{1/2} \lesssim \paren*{\paren*{1+\frac{\kappa B}{d}}^{2k+2} - 1}\frac{C}{\kappa d} + \paren*{1+\frac{\kappa B}{d}}^{2k} \hat{s}_0^2.
	\end{equation}
	Taking the ratio of \eqref{eq:lower_bound_for_sk_hat} and \eqref{eq:upper_bound_for_sk_hat}, plugging in the time point $k=T$, we get
	\begin{align*}
	\frac{\E\braces{\hat{s}_T^2}}{\E\braces{\hat{s}_T^4}^{1/2}} & \gtrsim \frac{\paren*{\paren*{1+\kappa AB/d}^{2T+2} - 1}/\kappa d + \paren*{1+\kappa AB/d}^{2T} \hat{s}_0^2}{\paren*{\paren*{1+\kappa B/d}^{2T+2} - 1}/\kappa d + \paren*{1+\kappa B / d}^{2T} \hat{s}_0^2} \\
	& = \frac{\paren*{1+\kappa AB/d}^{2T}}{\paren*{1+\kappa B/d}^{2T}} \cdot \frac{\paren*{1+\kappa AB/d}^2 + \kappa d \hat{s}_0^2 - \paren*{1+\kappa AB/d}^{-2T}}{\paren*{1+\kappa B/d}^2 + \kappa d \hat{s}_0^2 - \paren*{1+\kappa B/d}^{-2T}} \\
	& \asymp \paren*{\frac{1+\kappa AB/d}{1+\kappa B/d}}^{2T}.
	\end{align*}
	Note that in the last equation, we used the definition of $A$ and \eqref{eq:choice_of_T}.
	
	Since
	\begin{equation*}
	\frac{1+\kappa AB/d}{1+\kappa B/d} = \frac{d + \kappa B - \kappa B/\log^2 d}{d + \kappa B} = 1 - \frac{\kappa B}{(d +\kappa B)\log^2 d},
	\end{equation*}
	we may take logs and observe
	\begin{align*}
	\log \paren*{\frac{1+\kappa AB/d}{1+\kappa B/d}}^{2T} & \asymp \frac{d}{B}\log d \cdot \log\paren*{1 - \frac{\kappa B}{(d +\kappa B)\log^2 d}} \\
	& \asymp - \frac{d}{B}\log d \cdot \frac{B}{d\log^2 d} \\
	& = - \frac{1}{\log d}.
	\end{align*}
	
	Putting everything together, we have $\frac{\E\braces{\hat{s}_T^2}}{\E\braces{\hat{s}_T^4}^{1/2}} \gtrsim 1$. As such, for $d$ large enough, we may apply the Paley-Zygmund inequality to see that we may pick some $0 < \gamma_1 < \gamma$ such that $\abs{\hat{s}_T} \geq \gamma_1$, with probability at least some constant $p' > 0$. Let us now condition on the intersection of this event and that promised to us by Lemma \ref{lem:process_wise_stochastic_dominance} when we choose $k=T$. The probability of the intersection is at least
	\begin{equation*}
	p' - T\cdot\paren*{\frac{1}{d^2} + \frac{C}{\sqrt{B}}} = p' - \frac{C\log d}{Bd} - \frac{Cd\log d}{B^{3/2}} \geq p - \frac{C}{\sqrt{\log d}} \eqqcolon p
	\end{equation*} 
	where the inequality holds for $d$ large enough. On this event, we have $\tau_{2a} \leq  TB \leq Cd\log d$ as we wanted.
\end{proof}

\section{Phase 2b: Approximation by drift process} \label{sec:phase_2b}

In the previous two sections, we have bounded the duration of Phase 2a of the SGD process, that is, the time it takes for $\abs{s_k}$ to increase to a constant value. The goal of this section is to bound the duration of Phase 2b in which the iterates converge to the ``basin of convexity'' around $\boldy^*$ or $-\boldy^*$. We will prove the following theorem.

\begin{theorem}[Bound for duration of Phase 2b] \label{lem:bound_for_phase_2b}
	We have $\tau_{2b} - \tau_{2a} \leq Cd$ with probability at least $1-1/d$.
\end{theorem}

As mentioned in the overall proof outline, the idea is to understand the trajectory of the drift process, and then show that the fluctuations do not affect the trajectory by too much. For convenience, we condition on the event $\tau_{2a} < \infty$, and then use the strong Markov property to re-index, setting $\tau_{2a} = 0$.

The drift process $\braces{\bar{\boldy}^{(k)}}_k$ is defined via the deterministic update
\begin{equation*}
\bar{\boldy}^{(k+1)} \coloneqq \bar{\boldy}^{(k)} + (\bar{\alpha}(\bar{\boldy}^{(k)}), \bar{\beta}(\bar{\boldy}^{(k)})),
\end{equation*}
with the initialization $\bar{\boldy}^{(0)} = \boldy^{(0)}$. For simplicity, we denote $\bar{r}_k \coloneqq r(\bar{\boldy}^{(k)})$ and $\bar{s}_k \coloneqq s(\bar{\boldy}^{(k)})$.

\begin{lemma}[Behavior of drift sequence]
	Let $\tau = \min\braces*{k ~\colon~ \bar{s}_k \geq 1 - (\gamma_2 - \epsilon)/2}$ for some small $\epsilon > 0$. Then for $d$ large enough, we have $\tau \leq Cd$, where $C$ is a universal constant depending only on $\gamma_1$, $\gamma_2$, and $\epsilon$.
\end{lemma}

\begin{proof}
	 First, choose $\epsilon$ in Lemma \ref{lem:lower_bound_for_drift} to be equal to $\gamma_2/4$. Let $d$ be large enough so that $\abs{r_0^2-1} \leq C\frac{\log d}{\sqrt{d}} \leq \eta$, where $\eta = \eta(\epsilon)$ is the required value in Lemma \ref{lem:lower_bound_for_drift}. By Lemma \ref{lem:SGD_drift}, we see that $\abs{\bar{r}_k^2-1} \leq \eta$ for all $k \leq \tau$. This allows us to use Lemma 2.4 to observe that the recursive inequality $\bar{s}_{k+1} \geq (1+c/d)\bar{s}_k$ applies whenever $k \leq \tau$, where $c$ is a constant only depending on $\gamma_2$. By the definition of $\tau$, we have
	\begin{equation*}
	(1+c/d)^{\tau-1}\gamma_1 < 1 - (\gamma_2 - \epsilon)/2,
	\end{equation*}
	which we can solve to get 
	\begin{equation*}
	\tau \leq \frac{\log\paren*{(2-\gamma_2+\epsilon)/2\gamma_1}}{\log(1+c/d)} + 1 \leq Cd. \qedhere
	\end{equation*}
\end{proof}

\begin{lemma}[Approximation error]
	We have
	\begin{equation} \label{eq:approximation_error}
	\abs{s_{\tau}-\bar{s}_{\tau}} \leq \frac{C\log d}{\sqrt{d}}
	\end{equation}
	with probability at least $1-2/d$.
\end{lemma}

\begin{proof}
	Recall that by the discussion at the end of Section \ref{sec:uniform_control_of_r_k}, we have
	\begin{equation} \label{eq:bound_on_r_in_sec6}
	\sup_{0 \leq k \leq \tau} \abs{r_k^2-1} \leq \frac{C\log d}{\sqrt{d}}.
	\end{equation}
	Using Lemma \ref{lem:uniform_Bernstein}, there is also a probability $1-1/d$ event over which we have
	\begin{equation} \label{eq:phase_2b_uniform_bound_on_beta}
	\sup_{0 \leq t \leq \tau} \abs*{\sum_{k=0}^{t-1} \paren*{\beta_k(\boldy^{(k)})- \bar{\beta}(\boldy^{(k)})}} \leq C\sqrt{d}\log d.
	\end{equation}
	
	We will show that \eqref{eq:approximation_error} holds when conditioned on both of these events. First, for any $t \leq \tau$, we may write
	\begin{equation*}
	\bar{s}_t - \bar{s}_0 = \sum_{k=0}^{t-1} \paren*{ \bar{s}_{k+1} - \bar{s}_{k} } = \frac{1}{d}\sum_{k=0}^{t-1} \bar{\beta}(\bar{\boldy}^{(k)}),
	\end{equation*} 
	and similarly we have \eqref{eq:martingale_decomposition_for_s}, which we reproduce here:
	\begin{equation*}
	s_t - s_0 = \frac{1}{d}\sum_{k=0}^{t-1} \bar{\beta}(\boldy^{(k)}) + \frac{1}{d}\sum_{k=0}^{t-1} \paren*{\beta_k(\boldy^{(k)})- \bar{\beta}(\boldy^{(k)})}.
	\end{equation*}
	
	Subtracting these two equations, and recalling that $\bar{s}_0 = s_0$, we get
	\begin{align} \label{eq:recursive_bound_for_s_t-bar_s_t}
	\abs*{s_t-\bar{s}_t} & \leq \frac{1}{d}\abs*{\sum_{k=0}^{t-1} \paren*{\beta_k(\boldy^{(k)})- \bar{\beta}(\boldy^{(k)})}} + \frac{1}{d}\abs*{ \sum_{k=0}^{t-1} \paren*{ \bar{\beta}(\boldy^{(k)}) - \bar{\beta}(\bar{\boldy}^{(k)}) } } \nonumber \\
	& \leq \frac{C\log d}{\sqrt{d}} + \frac{1}{d}\sum_{k=0}^{t-1} \abs*{ \bar{\beta}(\boldy^{(k)}) - \bar{\beta}(\bar{\boldy}^{(k)})  }.
	\end{align}
	where the bound for the first term on the right hand side comes from \eqref{eq:phase_2b_uniform_bound_on_beta}
	
	Meanwhile, by Lemma \ref{lem:Lipschitz_continuity}, we also have
	\begin{align}
	\abs*{ \bar{\beta}(\boldy^{(t)}) - \bar{\beta}(\bar{\boldy}^{(t)}) } \leq L \paren*{ \abs*{r_{t} - \bar{r}_{t}} + \abs*{s_{t}-\bar{s}_{t}} }
	\end{align}
	and plugging in \eqref{eq:recursive_bound_for_s_t-bar_s_t} and \eqref{eq:bound_on_r_in_sec6} gives
	\begin{equation*}
	\abs*{ \bar{\beta}(\boldy^{(t)}) - \bar{\beta}(\bar{\boldy}^{(t)}) } \leq L \paren*{ \frac{C\log d}{\sqrt{d}} + \frac{1}{d}\sum_{k=0}^{t-1} \abs*{ \bar{\beta}(\boldy^{(k)}) - \bar{\beta}(\bar{\boldy}^{(k)})  }}.
	\end{equation*}
	
	We are now in a position to apply Lemma \ref{lem:recursive_inequality} with $x_t = \frac{1}{d}\abs*{ \bar{\beta}(\boldy^{(t)}) - \bar{\beta}(\bar{\boldy}^{(t)}) }$, $\rho = L/d$, and $\xi = C\log d/\sqrt{d}$. Doing so, we get
	\begin{align*}
	\frac{1}{d}\sum_{k=0}^{\tau-1} \abs*{ \bar{\beta}(\boldy^{(k)}) - \bar{\beta}(\bar{\boldy}^{(k)})  } & \leq \frac{C\log d}{\sqrt{d}} \cdot \paren*{(1+L/d)^{\tau} - 1} \\
	& \leq \frac{C \log d}{\sqrt{d}}
	\end{align*}
	as we wanted.
\end{proof}

\begin{proof}[Proof of Lemma \ref{lem:bound_for_phase_2b}]
	Set $\epsilon = \gamma_2/2$. Combining the previous two lemmas, we have
	\begin{align*}
	\abs{s_\tau} & \geq \abs{\bar{s}_\tau} - \abs{s_\tau - \bar{s}_\tau} \\
	& \geq 1 - \frac{\gamma_2 - \epsilon}{2} - \frac{C\log d}{\sqrt{d}},
	\end{align*}
	and combined with \eqref{eq:bound_on_r_in_sec6} gives
	\begin{align*}
	\Psi(\boldy^{(\tau)}) & = r_{\tau}^2 - 2\abs{s_\tau} + 1 \\
	& = r_{\tau}^2 - 1 + 2(1-\abs{s_\tau}) \\
	& \leq \frac{C\log d}{\sqrt{d}} + \gamma_2 - \epsilon \\
	& \leq \gamma_2,
	\end{align*}
	where the last inequality holds for $d$ large enough. As such, we have $\tau_{2b} \leq \tau \leq Cd$.
\end{proof}

\section{Linear convergence in Phase 3 and proof of Theorem \ref{thm:main_theorem}} \label{sec:proof_of_main_thm}

\begin{proof}[Proof of Theorem \ref{thm:main_theorem}]
	To summarize, we have showed that $\tau_1 \leq Cd\cdot(\log d + \log\abs{r_0^2-1})$ with probability at least $1-C/\log d$ (Lemma \ref{lem:bound_for_duration_of_phase1}), $\tau_{2a} - \tau_1 \leq Cd\log d$ with probability at least $p$ (Theorem \ref{thm:bound_tau_2a}), and $\tau_{2b} - \tau_{2a} \leq Cd$ with probability at least $1-1/d$ (Lemma \ref{lem:bound_for_phase_2b}). Putting all of these together gives
	\begin{align*}
	\tau_{2b} \leq  Cd\log \paren*{d \cdot \norm{\boldx^{(0)}}\vee 1}
	\end{align*}
	with probability at least $p - C/\log d - C/d$, which is larger than $p/2$ for $d$ large enough. Unfortunately, this is not good enough for our purposes and we need to do a bit more work to bring down the error probability.
	
	Let us condition on the $1-1/d$ probability event for which \eqref{eq:radius_uniform_concentration_pt2} holds so that we have uniform control over $\braces{r_k^2}$ over an appropriate timescale (more precisely, we use the coupling argument explained in the discussion after Lemma \ref{lem:uniform_radius_bound}). Define $A \coloneqq Cd\log \paren*{d \cdot C_3}$, where $C_3$ is the same constant used in \eqref{eq:radius_uniform_concentration_pt2}. Then by the strong Markov property, for any $k_0 > 0$, we have
	\begin{align*}
	\P\braces*{\tau_{2b} - \tau_1 \geq A + k_0 } & = \P\braces*{\inf_{0 \leq k < k_0 + A} \Psi(\boldy^{(\tau_1+k)}) > \gamma_2 } \\
	& = \E\braces*{ \P\braces*{\inf_{k_0 \leq k < k_0 + A} \Psi(\boldy^{(\tau_1+k)}) > \gamma_2 ~\vline~ \mathcal{F}_{\tau_1 + k_0}} \mathbf{1}\braces*{\inf_{0 \leq k < k_0} \Psi(\boldy^{(\tau_1+k)}) > \gamma_2 }} \\
	& \leq \paren*{1-p/2} \cdot \P\braces*{\tau_{2b} - \tau_1 \geq k_0 },
	\end{align*}
	which implies that
	\begin{equation*}
	\P\braces*{\tau_{2b} - \tau_1 \geq A + k_0 ~\vline~ \tau_{2b} - \tau_1 \geq k_0} \leq 1 - p/2.
	\end{equation*}
	
	As such, we have
	\begin{align*}
	\P\braces*{\tau_{2b} - \tau_1 \geq tA} & = \prod_{k=0}^{t-1} \P\braces*{\tau_{2b} - \tau_1 \geq (k+1)A ~\vline~ \tau_{2b} - \tau_1 \geq kA} \\
	& \leq (1-p/2)^t.
	\end{align*}
	If we set $t = \frac{\log(10)}{\log(1-p/2)}$, we see that with probability at least $0.9 - C/\log d - 1/d$,
	\begin{align*}
	\tau_{2b} & \leq tA + \tau_1 \\
	& \lesssim d\log d + d\log \paren*{d \cdot \norm{\boldx^{(0)}}\vee 1} \\
	& \leq d\cdot\paren*{\log d + \log \paren*{\norm{\boldx^{(0)}}\vee 1}}.
	\end{align*}
	As mentioned earlier, we will take $T = \tau_{2b}$.
	
	It remains to establish linear convergence of the process to $\pm\boldy^*$ during the iterations following $T$. To do this, we will follow the ideas of Section 3 in \cite{Tan2017}. First, we choose $\gamma_2$ in \eqref{eq:tau_2b} to be equal to $\pi^2/320$, and define a stopping time $\tau \coloneqq \min\braces{k \geq \tau_{2b} ~\colon~ \Psi(\boldy^{(k)}) \geq \pi^2/16}$. The stopping time argument in \cite{Tan2017} then implies that $\P\braces{\tau = \infty} \geq 0.95$. Also, we may use Lemma 2.2 therein to guarantee conditional contraction in each step. Namely, if $\Psi(\boldy^{(k)}) \leq \pi^2/16$ for any $k$, then
	\begin{equation*}
	\E\braces*{\Psi(\boldy^{(k+1)}) ~\vline~ \mathcal{F}_k} \leq \paren*{1-\frac{1}{2d}}\Psi(\boldy^{(k)}).
	\end{equation*}
	
	If we write $Y_k \coloneqq (1-1/2d)^{-k}\Psi(\boldy^{(T+k)})\cdot \mathbf{1}_{\tau > T+k}$, then the above bound allows us to compute
	\begin{align*}
	\E\braces*{Y_{k+1} ~\vline~ \mathcal{F}_{T+k}} & = \paren*{1-\frac{1}{2d}}^{-k-1}\E\braces*{\Psi(\boldy^{(T+k+1)}) ~\vline~ \mathcal{F}_{T+k}} \mathbf{1}_{\tau > {T+k}} \\
	& \leq \paren*{1-\frac{1}{2d}}^{-k-1}\paren*{1-\frac{c}{d}}\Psi(\boldy^{(T+k)})\mathbf{1}_{\tau > T+k} \\
	& \leq \paren*{1-\frac{1}{2d}}^{-k}\Psi(\boldy^{(T+k)})\mathbf{1}_{\tau > T + k-1} \\
	& = Y_k,
	\end{align*}
	and we see that $\braces{Y_k}_k$ is a supermartingale with respect to the filtration $\braces{\mathcal{F}_{T+k}}_k$.
	
	By the supermartingale inequality, there is a probability $0.95$ event over which
	\begin{equation*}
	\sup_{k \geq 0}Y_k \leq 20\Psi(\boldy^{(\tau_{2b})}) \leq 20\gamma_2.
	\end{equation*}
	On the intersection of this event, and that on which $\tau = \infty$, we have that
	\begin{equation*}
	\Psi(\boldy^{(T+k)}) \leq 20\gamma_2\paren*{1-\frac{1}{2d}}^{k} \leq \paren*{1-\frac{1}{2d}}^{k}
	\end{equation*}
	for all $k \geq 0$ as we wanted. If we total the measure of the excluded bad events, the final success probability is at least $0.8 - C/\log d - 1/d$ as promised.
\end{proof}

We now discuss some straightforward extensions of this result. First, the lack of a high probability guarantee is a little unfortunate, and results from having to apply the supermartingale inequality. We are unsure whether this can be overcome theoretically, but we can easily modify the algorithm so that convergence holds with high probability. To do this, we use the ``majority vote'' procedure described in \cite{Tan2017}. The price we have to pay is an additional $\log(1/\delta)$ factor on the number of iterations, where $\delta$ is the total error probability we can tolerate.

Second, in the algorithm we presented, we required fresh samples to be used in every update step. Once we are in the linear convergence regime, however, samples can actually be reused so long as we choose uniformly from $\Omega(d)$ of them (see \cite{Tan2017}).

\section{Conclusion and discussion} \label{sec:conclusion}

In this paper, we have analyzed the convergence of constant step-size stochastic gradient descent for the non-convex, non-smooth phase retrieval objective \eqref{eq:intensities_squared_loss}, for which we assume Gaussian sampling vectors, and use an arbitrary initialization. The main idea was to view the SGD sequence as a Markov chain on a summary state space, and then use the natural $1/d$ step size scaling to argue that as $d$ tends to infinity, the process trajectory converges to something we understand. We believe that our proof framework and techniques will have applications beyond the vanilla phase retrieval model, and indeed inform the theory of non-convex optimization in general.

\subsection{Noisy measurements and non-Gaussian measurements}

We have analyzed phase retrieval in the noiseless setting, as is customary in the literature. It is easy to see that the arguments still go through for an additive noise model, except that we may now need to either use batch updates, or reduce the step size further, in order to get convergence. It will also be interesting to see whether the results we have can be extended to the setting of non-Gaussian sampling vectors. Our state space argument relies on the rotational symmetry of the sampling vectors, and this clearly will not hold in the non-Gaussian setting. However, it may be possible for this to be overcome using approximations.

\subsection{Extensions to other single index models}

Phase retrieval is an example of a single index model with the link function $f(t) = \abs{t}$. One can easily check that the state space argument generalizes to models with other link functions. Less clear is how to generalize the other arguments in the paper. Nonetheless, we expect this to be not too difficult, so long as we assume some natural regularity conditions on the link function. Finally, we conjecture that similar ideas can work for analyzing SGD for low-rank matrix sensing, since heuristically what makes everything work is the underlying low-dimensional structure in the problem.

\subsection{Nonconvex optimization}

In the introduction to this paper, we have already talked about the growing interest in a theoretical understanding of first-order methods applied to non-convex problems. Here, we reiterate that the main contribution of our paper should be seen as not simply providing a convergence guarantee for SGD, but also one that has close to optimal sample and computational complexity. We are able to achieve this through a careful analysis of the ``essential dynamics'' of the SGD process, as represented by the summary state space.

Standard proofs of convergence for first order methods applied to convex problems proceed by tracking one of the following three quantities: $\norm{\boldx_k - \boldx^*}^2$, $f(\boldx_k) - f(\boldx^*)$, or $\norm{\grad f(\boldx^*)}$. Under our framework, this can be seen as implicitly using a one-dimensional state space. In non-convex optimization, however, it makes sense to use a multi-dimensional state space, whereby we measure ``progress'' in terms of multiple quantities. The number of such quantities one needs to track for a given problem can then perhaps be used to define a notion of ``complexity'' for that problem.

\subsection*{Acknowledgements}

Y.T. was partially supported by NSF CCF-1740855, and completed part of this manuscript while visiting the Simons Institute for the Theory of Computing. R.V. was partially supported by U.S. Air Force Grant FA9550-18-1-0031. Y.T. would like to thank Xiang Cheng, Jelena Diakonikolas, Michael Jordan, and Yian Ma for helpful discussions.

\nocite{*}
\bibliographystyle{acm}
\bibliography{Projects-PR_random_init}

\appendix

\section{State space calculations} \label{sec:state_space_calculations}

Recall that we let $\boldx^{(0)},\boldx^{(1)},\boldx^{(2)},\ldots$ denote the sequence obtained by iteratively performing the SGD update \eqref{eq:SGD_update} with constant step size $\eta = \frac{1}{d}$.

\begin{lemma}
	The sequence $\boldx^{(0)},\boldx^{(1)},\boldx^{(2)},\ldots$ is a Markov chain whose transition kernel has the random mapping representation
	\begin{equation}
	\boldx^{(k+1)} = \boldx^{(k)} + \frac{1}{d}\Delta(\boldx^{(k)}),
	\end{equation}
	where
	\begin{equation} \label{eq:defn_of_Delta}
	\Delta(\boldx) = \paren*{\inprod{\bolda,\boldx^*-\boldx} - 2\mathbf{1}_{\mathcal{A}}\cdot\inprod{\bolda,\boldx^*}}\bolda,
	\end{equation}
	$\bolda \sim \Unif(\sqrt{d}S^{d-1})$, and $\mathcal{A}$ is the event that $\sign\paren{\inprod{\bolda,\boldx}} \neq \sign\paren{\inprod{\bolda,\boldx^*}}$.
\end{lemma}

\begin{proof}
	The fact that the sequence is a Markov chain is clear. After dropping indices, we may write the update step \eqref{eq:SGD_update} as
	\begin{align*}
	\Delta(\boldx) & \coloneqq \paren*{\sign(\inprod{\bolda,\boldx})\abs{\inprod{\bolda,\boldx^*}}-\inprod{\bolda,\boldx}}\bolda \\
	& = \paren*{\sign(\inprod{\bolda,\boldx})\sign(\inprod{\bolda,\boldx^*})\inprod{\bolda,\boldx^*}-\inprod{\bolda,\boldx}}\bolda \\
	& = \paren*{\inprod{\bolda,\boldx^*-\boldx} - 2\mathbf{1}\braces{\inprod{\bolda,\boldx}\inprod{\bolda,\boldx^*} <0}\inprod{\bolda,\boldx^*}}\bolda
	\end{align*}
\end{proof}

\begin{proof}[Proof of Theorem \ref{thm:state_space_MC}]
Abusing notation slightly, let us define the state space updates
\begin{equation*}
\alpha(\boldx) \coloneqq d\cdot\paren*{r^2(\boldx+\frac{1}{d}\Delta(\boldx)) - r^2(\boldx)}
\end{equation*}
\begin{equation*}
\beta(\boldx) \coloneqq d\cdot \paren*{s(\boldx+\frac{1}{d}\Delta(\boldx)) - s(\boldx)}.
\end{equation*}

Our goal is to compute formulas for the distributions for $\alpha(\boldx)$ and $\beta(\boldx)$, in particular showing that they depend on $\boldx$ only through $\boldy$. To this end, we simplify notation, denoting $r = r(\boldx)$, $s = s(\boldx)$, and $\theta = \theta(\boldx)$. Furthermore, define $\boldx^\perp \coloneqq \frac{\boldP_{\boldx^*}^\perp\boldx}{\norm{\boldP_{\boldx^*}^\perp\boldx}}$.

We can decompose $\boldx$ into its components parallel and perpendicular to $\boldx^*$, writing
\begin{equation} \label{eq:formula_for_x}
\boldx = r\cos\theta \boldx^* + r\sin\theta \boldx^\perp.
\end{equation}
Let $\bolda \sim \Unif(\sqrt{d}S^{d-1})$ be the random vector used to generate $\Delta(\boldx)$. We also have the orthogonal decomposition
\begin{equation} \label{eq:formula_for_u}
\bolda = a_1 \boldx^* + a_2 \boldx^\perp + \boldr
\end{equation}
where $a_1$, $a_2$ are the marginals of $\bolda$ along $\boldx^*$ and $\boldx^\perp$ respectively, and $\boldr$ is defined as the remainder in the decomposition above.

Combining these formulas, we immediately get
\begin{equation} \label{eq:formula_for_x_dot_u}
\inprod{\bolda,\boldx} = r\cos\theta a_1 + r\sin\theta a_2.
\end{equation}

To compute the formula for $\alpha(\boldx)$, we first expand
	\begin{align} \label{eq:expanding_alpha}
	\alpha(\boldx) & = d\cdot \paren*{r^2(\boldx + \Delta(\boldx)/d) - r^2(\boldx)} \nonumber \\
	& = d \cdot\paren*{\norm{\boldx+\Delta(\boldx)/d}^2 - \norm{\boldx}^2} \nonumber \\
	& = 2\inprod{\boldx,\Delta(\boldx)} + \norm{\Delta(\boldx)}^2/d.
	\end{align}
	Now use equations \eqref{eq:defn_of_Delta}, \eqref{eq:formula_for_u} and \eqref{eq:formula_for_x_dot_u} to write
	\begin{equation*}
	\inprod{\Delta(\boldx),\boldx} = \paren*{\inprod{\bolda,\boldx^*-\boldx} - 2\mathbf{1}_{\mathcal{A}} \cdot \inprod{\bolda,\boldx^*}}\inprod{\bolda,\boldx},
	\end{equation*}
	and
	\begin{align} \label{eq:squared_norm_for_Delta}
	\frac{\norm{\Delta(\boldx)}^2}{d} & = \frac{\norm{\bolda}^2}{d}\cdot \paren*{\inprod{\bolda,\boldx^*-\boldx}^2 + 4\inprod{\bolda,\boldx^*}^2\mathbf{1}_{\mathcal{A}} - 4\inprod{\bolda,\boldx^*-\boldx}\inprod{\bolda,\boldx^*}\mathbf{1}_{\mathcal{A}}} \nonumber \\
	& = \inprod{\bolda,\boldx^*-\boldx}^2 + 4\inprod{\bolda,\boldx}\inprod{\bolda,\boldx^*}\mathbf{1}_{\mathcal{A}}
	\end{align}
	
	Plugging these into \eqref{eq:expanding_alpha}, we have
	\begin{align*}
	\alpha(\boldx) & = 2\inprod{\bolda,\boldx^*-\boldx}\inprod{\bolda,\boldx} +  \inprod{\bolda,\boldx^*-\boldx}^2 \\
	& = \inprod{\bolda,\boldx^*}^2-\inprod{\bolda,\boldx}^2 \\
	& = a_1^2 - \paren*{r\cos\theta a_1 + r\sin\theta a_2}^2  \\
	& = (1-r^2\cos^2\theta) a_1^2-r^2\sin^2\theta a_2^2 +r^2\sin\theta\cos\theta a_1a_2
	\end{align*}
	
	For the second formula, we note that $\beta(\boldx) = \inprod{\Delta(\boldx),\boldx^*}$, and write
	\begin{align*}
	\beta(\boldx) & = \paren*{\inprod{\bolda,\boldx^*-\boldx} - 2\mathbf{1}_{\mathcal{A}} \cdot\inprod{\bolda,\boldx^*}}\inprod{\bolda,\boldx^*} \\
	& = \paren*{a_1(1-r\cos\theta) - a_2r\sin\theta - 2\textbf{1}_{\mathcal{A}}a_1}a_1 \\
	& = (1-r\cos\theta)a_1^2 - r\sin\theta a_1a_2 - 2\textbf{1}_{\mathcal{A}}a_1^2.
	\end{align*}
	Here, the second equality follows from a combination of \eqref{eq:formula_for_u}, \eqref{eq:formula_for_x} and \eqref{eq:formula_for_x_dot_u}.
\end{proof}

\begin{lemma}[Expectations involving $\mathcal{A}$]
	The distribution of the event $\mathcal{A} = \mathcal{A}(\theta)$ depends only on the angle $\theta$ between $\boldx$ and $\boldx^*$. Furthermore, using the formula \eqref{eq:formula_for_u}, we have the following identities.
	\begin{equation}
	\P\braces{\mathcal{A}(\theta)} = \frac{\theta}{\pi}
	\end{equation}
	\begin{equation} \label{eq:angle_identity_1}
	\E\braces{a_1^2 \mathbf{1}_{\mathcal{A}(\theta)}} = \frac{1}{2\pi}\paren*{2\theta-\sin(2\theta)}
	\end{equation}
	\begin{equation} \label{eq:angle_identity_2}
	\E\braces{a_1a_2 \mathbf{1}_{\mathcal{A}(\theta)}} = \frac{1}{2\pi}\paren*{\cos(2\theta)-1}
	\end{equation}
\end{lemma}

\begin{proof}
	For the first identity, let $\tilde{\bolda} \coloneqq (a_1,a_2)$. We may write $a_1 = \norm{\tilde{\bolda}}\cos t$ and $a_2 = \norm{\tilde{\bolda}}\sin t$. Event $\mathcal{A}$ occurs precisely when $\frac{\pi}{2} \leq t \leq \frac{\pi}{2} + t$.
	
	To prove \eqref{eq:angle_identity_1}, observe that $t$ is independent of $\norm{\tilde{\bolda}}$. This allows us to compute:
	\begin{align*}
	\E\braces{a_1^2 \mathbf{1}_\mathcal{A}} & = \E\braces{(a_1^2+a_2^2)}\E\braces{ \cos^2(t) \mathbf{1}_\mathcal{A}} \\
	& = 2\cdot\frac{1}{\pi} \int_{\frac{\pi}{2}}^{\frac{\pi}{2}+\theta} \cos^2(t)dt \\
	&= \frac{2}{\pi} \int_{0}^{\theta} \sin^2(t)dt \\
	& = \frac{1}{\pi} \int_0^\theta 1 - \cos(2t) dt \\
	& = \frac{1}{\pi} \sqbracket*{t-\frac{\sin(2t)}{2}}_0^\theta \\
	& = \frac{1}{\pi}\paren*{\theta-\frac{\sin(2\theta)}{2}}.
	\end{align*}
	A similar calculation yields \eqref{eq:angle_identity_2}.
\end{proof}

\begin{proof}[Proof of Lemma \ref{lem:SGD_drift}]
	It is easy to compute $\E u^2 = 1$, and $\E uv = 0$. Applying this and \eqref{eq:angle_identity_1} in the previous lemma to \eqref{eq:formula_for_alpha} and \eqref{eq:formula_for_beta} yields \eqref{eq:expectation_of_alpha} and \eqref{eq:expectation_of_beta}.
\end{proof}

\section{Properties of subexponential random variables} \label{sec:properties_of_subexponentials}

Subexponential random variables are defined in terms of tail bounds, and can also be equivalently defined as elements of an Orlicz space with Orlicz norm
\[
\norm{X}_{\psi_1} \coloneqq \inf\braces{C \colon \E\exp(X/C) \leq 2}.
\]

One may easily check that this is a norm, which allows for easy tail bounds for random variables that are sums of subexponential random variables. We will only state the propreties of subexponential variables needed in our paper, and refer the interested reader to the textbook \cite{Vershynin}.

\begin{lemma}[$\psi_1$-norm for $\alpha(\boldy)$ and $\beta(\boldy)$] \label{lem:psi_1_norm_bounds}
	Let $\mathcal{D} \subset \mathcal{Y}$ be a compact domain. Then the subexponential norms of $\alpha(\boldy)$ and $\beta(\boldy)$ are uniformly bounded for $\boldy \in \mathcal{Y}$.
\end{lemma}

\begin{proof}
	Suppose we have $r(\boldy) \leq R$ for all $\boldy \in \mathcal{Y}$. Then
	\begin{align*}
	\norm{\alpha(\boldy)}_{\psi_1} & = \norm{(1-r^2\cos^2\theta) \cdot u^2-r^2\sin^2\theta \cdot v^2 +r^2\sin\theta\cos\theta \cdot uv}_{\psi_1} \\
	& \leq \abs{1-r^2\cos^2\theta} \cdot\norm{u^2}_{\psi_1} +r^2\sin^2\theta \cdot\norm{v^2}_{\psi_1} +\abs{r^2\sin\theta\cos\theta} \cdot\norm{uv}_{\psi_1} \\
	& \leq \paren{R^2+1} \cdot \paren*{\norm{u^2}_{\psi_1} + \norm{v^2}_{\psi_1} + \cdot\norm{uv}_{\psi_1}} \\
	& \leq \paren{R^2+1} \cdot \paren*{\norm{u}_{\psi_2}^2 + \norm{v}_{\psi_2}^2 + \cdot\norm{u}_{\psi_2}\norm{v}_{\psi_2}} \\
	& \lesssim R^2 + 1.
	\end{align*}
	The first inequality is an application of the triangle inequality, the third follows from the following basic property for subexponential random variables: $\norm{XY}_{\psi_1} \leq \norm{X}_{\psi_2}\norm{Y}_{\psi_2}$. Finally, it is easy to check that $\norm{u}_{\psi_2} \lesssim 1$.
\end{proof}

\section{Lemmas for drift} \label{sec:drift_lemmas}

\begin{proof}[Proof of Lemma \ref{lem:lower_bound_for_drift}]
	For each fixed $r$, $\bar{\beta}$ is an odd function with respect to $s$, so it suffices to prove the statement for $s>0$. Let us now compute its derivative with respect to $s$. Differentiating \eqref{eq:expectation_of_beta} first with respect to $\theta$, we get
	\begin{align*}
	\frac{d}{d\theta}\bar{\beta}(r^2,s) & = r\sin\theta - \frac{1}{\pi}\paren*{2-2\cos(2\theta)} \\
	& = r\sin\theta - \frac{4}{\pi}\sin^2\theta.
	\end{align*}
	Next, note that since $s = r\cos\theta$, we have
	\begin{equation*}
	\frac{d\theta}{ds} = \paren*{\frac{ds}{d\theta}}^{-1} = -\frac{1}{r\sin\theta}
	\end{equation*}
	Using the chain rule and then simplifying, we thereby get
	\begin{align*}
	\frac{d}{ds}\bar{\beta}(r^2,s) & = \paren*{r\sin\theta - \frac{4}{\pi}\sin^2\theta} \cdot \paren*{-\frac{1}{r\sin\theta}} \\
	& = \frac{4}{\pi} \frac{\sqrt{r^2-s^2}}{r^2} - 1.
	\end{align*}
	
	From this expression, we can see that for any fixed $r > 0$, the function $s \mapsto \bar{\beta}(r^2,s)$ is concave downwards on $[0,r]$, with $\bar{\beta}(r^2,0) = 0$, and $\partial_s\bar{\beta}(r^2,s)\vline_{s=0} > 0$. This implies that the graph of $\bar{\beta}(r^2,s)$ as a function of $s$ lies beneath the line passing through the origin with slope $\partial_s \bar{\beta}(r^2,s)\vline_{s=0}$. When $r \geq \frac{1}{2}$, we have
	\begin{equation*}
	\partial_s \bar{\beta}(r^2,s)\vline_{s=0} = \frac{4}{\pi r} - 1 < \frac{8}{\pi} - 1
	\end{equation*}
	thereby giving the upper bound.
	
	For the lower bound, first note that concavity also implies that $\bar{\beta}(r^2,s) \geq \frac{s}{s'}\bar{\beta}(r^2,s')$ for any $0 < s < s' < r$. Now, one may easily check that $\bar{\beta}(1,1) = 0$, so that $\bar{\beta}(1,1-\epsilon) > 0$ for $\epsilon$ small enough. By continuity, there is some $\eta > 0$ for which
	\begin{equation*}
	\inf_{\abs{r^2-1} \leq \eta} \bar{\beta}(r^2,1-\epsilon) \geq \frac{\bar{\beta}(1,1-\epsilon)}{2}.
	\end{equation*}
	
	Indeed, since $\partial_r \bar{\beta}(r^2,s) =  -\cos\theta$, one may even provide a precise formula if one wishes. Set $\underline{b} \coloneqq \frac{\bar{\beta}(1,1-\epsilon)}{2(1-\epsilon)}$. This is the universal constant we want.
\end{proof}

\begin{lemma}[Lipschitz continuity] \label{lem:Lipschitz_continuity}
	Fix $\epsilon > 0$ in the previous lemma, and by making $\epsilon$ and $\eta$ smaller if necessary, assume that $\abs{s} < r - \epsilon/2$ for all $(r^2,s) \in \mathcal{D}$, where $\mathcal{D} \coloneqq \braces{(r^2,s) \in \mathcal{Y} ~\colon~ \abs{s} \leq 1-\epsilon, \abs{r^2-1} \leq \eta}$. Then $\bar{\beta}$ is Lipschitz continuous on $\mathcal{D}$ with Lipschitz constant bounded by a universal constant $L$ depending only on $\epsilon$.
\end{lemma}

\begin{proof}
	We have
	\begin{align*}
		\abs*{\bar{\beta}(\boldy) - \bar{\beta}(\boldy')} & \leq \abs*{s-s'} + \frac{2}{\pi}\abs{\theta-\theta'} + \frac{1}{\pi} \abs*{\sin(2\theta) - \sin(2\theta')} \\
		& \leq \abs*{s-s'} + \frac{4}{\pi}\abs{\theta-\theta'}.
	\end{align*}
	The first term is trivially bounded by $\norm{\boldy-\boldy'}$. Next, observe that $\theta = \arccos(s/\sqrt{r^2})$, which is jointly differentiable in $r^2$ and $s$, and so is Lipschitz continuous with respect to these coordinates on a compact set bounded away from $r=s$. By assumption, $\mathcal{D}$ is such a compact set.
\end{proof}

\section{Facts about Kolmogorov distance and stochastic dominance}

\begin{lemma}[Properties of Kolmogorov distance] \label{lem:kolmogorov_dist_properties}
	For any real-valued random variables $X$ and $Y$, we have
	\begin{equation}
	\norm{F_{X^2}-F_{Y^2}}_\infty \leq 2\norm{F_X-F_Y}_\infty.
	\end{equation}
	In addition, for any constant $c \in \R$, we have
	\begin{equation}
	\norm{F_{X+c}-F_{Y+c}}_\infty = \norm{F_X-F_Y}_\infty
	\end{equation}
\end{lemma}

\begin{proof}
	For any $t \geq 0$, we have
	\begin{align*}
	\abs{F_{X^2}(t)-F_{Y^2}(t)} & = \P\braces{X^2 \leq t} - \P\braces{Y^2 \leq t} \\
	& = \P\braces{-t \leq X \leq t} - \P\braces{-t \leq Y \leq t} \\
	& = F_X(t) - F_X(-t) + F_Y(t) - F_Y(-t) \\
	& \leq \abs{F_X(t)-F_Y(t)} + \abs{F_X(-t) - F_Y(-t)} \\
	& \leq 2\norm{F_X-F_Y}_\infty
	\end{align*}
	as we wanted. The second identity can be obtained similarly.
\end{proof}

\begin{lemma}[Characterization of stochastic dominance in terms of CDFs] \label{lem:characterization_by_CDFs}
	Let $X$ and $Y$ be real-valued random variables. Then for any $0 < \delta < 1$, $X$ stochastically dominates $Y$ up to error $\delta$ if and only if their CDFs satisfy $F_X \leq F_Y + \delta$.
\end{lemma}

\begin{proof}
	The forward implication is trivial, so we only need to show the backward implication. Let $q_X$ and $q_Y$ denote the quantile functions for $X$ and $Y$ respectively. We claim that for any $0 \leq t \leq 1-\delta$, we have $q_X(t+\delta) \geq q_Y(t)$. To see this, first recall the definition $q_Y(t) = \inf\braces{a ~\colon~ F_Y(a) \geq t}$. By definition, we thus have a decreasing sequence $x_n \downarrow q_X(t+\delta)$ such that $F_X(x_n) \geq t + \delta$ for each $n$. Then by assumption, $F_X(a) \leq F_Y(a) + \delta$ for any $a \in \R$, so
	\begin{equation*}
	F_Y(x_n) \geq F_X(x_n) - \delta \geq (t+\delta) - \delta = t,
	\end{equation*}
	which implies that
	\begin{equation*}
	q_Y(t) \leq \liminf_{n\to\infty}x_n = q_X(t+\delta)
	\end{equation*}
	as claimed.
	
	Let $U$ be uniformly distributed on $[0,1]$, and on the same probability space, define $U'(\omega) = U(\omega) + \delta \mod 1$. Then $U'$ is also uniformly distributed on $[0,1]$. Next, it easy to check that $q_X(U') \sim F_X$ and $q_Y(U) \sim F_Y$ (this is a standard construction in probability theory). With probability $1-\delta$, we have $\delta \leq U' = U + \delta \leq 1$. Conditioning on the event in which this occurs, we have
	\begin{equation*}
	q_X(U') = q_X(U+\delta) \geq q_Y(U),
	\end{equation*}
	which gives us the coupling we want.
\end{proof}

\begin{lemma}[Transitivity of stochastic dominance] \label{lem:transitivity_of_stochastic_dominance}
	Let $X$, $Y$, and $Z$ be random variables, $\delta, \delta' > 0$ such that $X \succeq_{\delta} Y$ and $Y \succeq_{\delta'} Z$. Then $X \succeq_{\delta+\delta'} Z$.
\end{lemma}

\begin{proof}
	This follows from the characterization in the previous lemma.
\end{proof}

\begin{lemma}[Preservation of stochastic dominance under monotone transformations] \label{lem:monotone_transformations}
	Let $X$ and $Y$, and $Z$ be random variables, $\delta > 0$ such that $X \succeq_{\delta} Y$. Suppose $X$ and $Y$ have range $R \subset \R$, and $\tau\colon R \to R$ is a non-decreasing transformation. Then $\tau(X) \succeq_{\delta} \tau(Y)$.
\end{lemma}

\begin{proof}
	This is obvious.
\end{proof}

\begin{corollary}[Kolmogorov distance and coupling] \label{cor:CDF_distance_and_coupling}
	Let $X$ and $Y$ be real-valued random variables such that their CDFs satisfy $\norm{F_X-F_Y}_\infty \leq \delta$ for some $0 < \delta < 1$. Then $X$ stochastically dominates $Y$ up to error $\delta$.
\end{corollary}

\begin{lemma}[Stochastic dominance ordering for truncated Gaussians] \label{lem:stochastic_dominance_for_truncated_Gaussians}
	The following hold:
	\begin{enumerate}
		\item Fix $\sigma^2, \epsilon > 0$. For any $b > 0$, define the random variable $X_b \coloneqq \rho_{\epsilon}\sqbracket{b +\sigma g}^2$. Then whenever $b' > b$, we have $X_{b'} \succeq X_{b}$.
		\item Fix $\sigma_0^2$, set $\sigma^2 = \frac{B\sigma_0^2}{d^2}$. Let $\epsilon(s)$ and $b(s)$ be defined as in \eqref{eq:defn_of_epsilon(s)} and \eqref{eq:defn_of_b(s)} respectively	.
		Consider the collection of random variables $Y_s \coloneqq \rho_{\epsilon(s)}\sqbracket{b(s) + \sigma g}^2$ for $-1/2 < s < 1/2$. For $d$ large enough, whenever $1/2 > s' > s > 0$, we have $Y_{s'} \succeq Y_{s}$.
	\end{enumerate}
	
\end{lemma}

\begin{proof}
	Throughout, we let $g$ denote a standard normal random variable. We start by proving the first statement. Let $F_b$ denote the CDFs of $X_b$ and $X_{b'}$ respectively. For any given $a > 0$, we wish to show that $F_b(a) \geq F_{b'}(a)$, and it suffices to show that $\frac{d}{db}F_b(a) \leq 0$ for $b > 0$. In order to do this, we write out $F_s$ in terms of the Gaussian CDF. We have
	\begin{align*}
	F_b(a) & = \P\braces{X_s \leq a} \\
	& = \P\braces{ -\sqrt{a} -\epsilon \leq \sigma g + b \leq \sqrt{a} + \epsilon} \\
	& = \P\braces*{ g \leq \frac{\sqrt{a}+\epsilon-b(s)}{\sigma}} - \P\braces*{g \leq \frac{-\sqrt{a}-\epsilon-b}{\sigma}} \\
	& = \Phi\paren*{\frac{\sqrt{a}+\epsilon-b}{\sigma}} - \Phi\paren*{\frac{-\sqrt{a}-\epsilon-b}{\sigma}}.
	\end{align*}
	Hence, we have
	\begin{equation*}
	\frac{d}{db}F_b(a) = \frac{1}{\sigma}\paren*{\phi\paren*{\frac{-\sqrt{a}-\epsilon-b}{\sigma}} - \phi\paren*{\frac{\sqrt{a}+\epsilon-b}{\sigma}}}.
	\end{equation*}
	Since $\phi$ is even, and decreases away from $0$, and $\abs{-\sqrt{a}-\epsilon-b} > \abs{\sqrt{a}+\epsilon-b}$ for $b>0$, the quantity on the right hand side is negative as we wanted.
	
	The second statement is proved similarly. We let $F_s$ denote the CDF of $Y_s$ and will show that $\frac{d}{ds} F_s(a) \leq 0$. As before we compute
	\begin{equation*}
	F_s(a) = \Phi\paren*{\frac{\sqrt{a}+\epsilon(s)-b(s)}{\sigma}} - \Phi\paren*{\frac{-\sqrt{a}-\epsilon(s)-b(s)}{\sigma}}.
	\end{equation*}
	Hence, we have
	\begin{align*}
	\frac{d}{ds}F_s(a) = \frac{\epsilon'(s)-b'(s)}{\sigma}\phi\paren*{\frac{\sqrt{a}+\epsilon(s)-b(s)}{\sigma}} - \frac{-\epsilon'(s)-b'(s)}{\sigma}\phi\paren*{\frac{-\sqrt{a}-\epsilon(s)-b(s)}{\sigma}}.
	\end{align*}
	
	For $0 < s < \sqrt{\frac{\log d}{B}}$, we have $\epsilon'(s) = 0$, and it is clear that this quantity is nonpositive. For $s > \sqrt{\frac{\log d}{B}}$, first observe that for any $x, y$, we have $e^{-(x-y)^2}/e^{-(x+y)^2} = e^{4xy}$. As such, we compute
	\begin{align*}
	\phi\paren*{\frac{\sqrt{a}+\epsilon(s)-b(s)}{\sigma}}\Big/ \phi\paren*{\frac{-\sqrt{a}-\epsilon(s)-b(s)}{\sigma}} & = \exp\paren*{\frac{2(\sqrt{a}+\epsilon(s))b(s)}{\sigma^2}} \\
	& \geq \exp\paren*{\frac{d^2}{B\sigma_0^2}\cdot \frac{CB^2 \log d}{d^2} \cdot s^2} \\
	& \geq \exp\paren*{\frac{d^2}{B\sigma_0^2}\cdot \frac{CB^2 \log d}{d^2} \cdot \frac{\log d}{B}} \\
	& = \exp\paren*{\frac{C\log^2 d}{\sigma_0^2}}.
	\end{align*}
	
	For $d$ large enough, this ratio is greater than
	\begin{equation*}
	\frac{b'(s)+\epsilon'(s)}{b'(s)-\epsilon'(s)} = \frac{1+ \kappa B/d + CB^2\log d /d^2}{1 + \kappa B/d - CB^2\log d / d^2},
	\end{equation*}
	which converges to $1$ as $d$ tends to infinity. This concludes the proof of the claim.
\end{proof}

\section{Combinatorial lemmas}

\begin{lemma} \label{lem:shrinking_partial_sums}
	Let $x_1,x_2,\ldots$ be a sequence of real numbers. Denote the partial sums by $s_t = \sum_{i=1}^t x_i$. For any $0 < \rho < 1$, define the sequence $w_1,w_2,\ldots$ via the recursive formula
	\begin{equation*}
	w_{t+1} \coloneqq \rho w_t + x_{t+1}.
	\end{equation*}
	If there is some positive integer $M$ and some $C > 0$ such that $\abs{s_t} \leq C$ for all $t \leq M$, then we also have $\abs{w_t} \leq 2C + \abs{w_0}$ for all $t \leq M$.
\end{lemma}

\begin{proof}
	We first prove by induction that the following representation for $w_t$ holds:
	\begin{equation} \label{eq:formula_for_w_t}
	w_t = s_t - (1-\rho)\sum_{i=1}^{t-1} \rho^{t-i-1} s_i + \rho^t w_0.
	\end{equation}
	First assume that the formula holds for some $t$. Then starting from the definition, we use the inductive hypothesis to write
	\begin{align*}
	w_{t+1} & = \rho w_t + x_{t+1} \\
	& = \rho \cdot \paren*{s_t - (1-\rho)\sum_{i=1}^{t-1} \rho^{t-1-i} s_i + \rho^t w_0} + x_{t+1} \\
	& = \rho s_t - (1-\rho)\sum_{i=1}^{t-1} \rho^{t-i} s_i + \rho^{t+1} w_0 + (s_{t+1} - s_t) \\
	& = s_{t+1} - (1-\rho)\sum_{i=1}^t \rho^{t-i} s_i + \rho^{t+1}w_0.
	\end{align*}
	
	Applying the triangle inequality to \eqref{eq:formula_for_w_t}, we get
	\begin{align*}
	\abs{w_t} & \leq \abs{s_t} + (1-\rho)\sum_{i=1}^t \rho^{i-1}\abs{s_{t-i}} + \rho^t \abs{w_0}\\
	& \leq C + (1-\rho)\sum_{i=1}^\infty \rho^{i-1}C + \abs{w_0} \\
	& \leq 2C + \abs{w_0}
	\end{align*}
	as we wanted.
\end{proof}

\begin{lemma} \label{lem:recursive_inequality}
	Let $x_1,x_2,\ldots$ be a sequence of non-negative real numbers. Denote the partial sums by $s_t = \sum_{i=1}^t x_i$. Let $\rho > 0$ and $\xi > 0$ be such that we have the recursive inequality
	\begin{equation*}
	x_t \leq \rho(\xi + s_{t-1}).
	\end{equation*}
	Then for all $t$, we have
	\begin{equation*}
	s_t \leq \xi\cdot\paren*{(1+\rho)^{t+1}-1}.
	\end{equation*}
\end{lemma}

\begin{proof}
	We have
	\begin{align*}
	s_t & = s_{t-1} + x_t \\
	& \leq s_{t-1} + \rho(\xi + s_{t-1}) \\
	& \leq (1+\rho)s_{t-1} + \rho \xi.
	\end{align*}
	Solving this recursion, we get
	\begin{equation*}
	s_t \leq \rho \xi \cdot \sum_{i=0}^{t-1} (1+\rho)^i = \xi\cdot\paren*{(1+\rho)^{t+1}-1}. \qedhere
	\end{equation*} 
\end{proof}

\section{Concentration inequalities} \label{sec:uniform_Bernstein}

\begin{proof}[Proof of Lemma \ref{lem:uniform_Bernstein}]
	This essentially follows from Theorem 1 in \cite{Howard2018}. For completeness and clarity, however, we give a direct proof of this here using the same technique.
	
	First fix $\epsilon > 0$. Let $\lambda \leq \frac{1}{2K}$, and for each positive integer $t$, we write
	\begin{equation*}
	L_t(\lambda) \coloneqq \exp(\lambda S_t - \lambda^2K^2t).
	\end{equation*}
	Observe that
	\begin{align*}
	\E\braces{L_{t+1}(\lambda) ~\vline~ \mathcal{G}_t} = \E\braces{e^{\lambda X_{t+1} - \lambda^2K^2} ~\vline~ \mathcal{G}_t}L_t(\lambda) & \leq L_t(\lambda),
	\end{align*}
	so that the sequence $L_t(\lambda)$ forms a supermartingale, and remains so if we extend this to time 0 by setting $L_0(\lambda) = 1$. We may now apply the supermartingale inequality. We have
	\begin{align} \label{eq:supermartingale_argument}
	\P\braces{\exists t \leq M ~\colon~ S_t \geq \lambda K^2 M + \epsilon/2 } & \leq \P\braces{\exists t \leq M ~\colon~ S_t \geq \lambda K^2 t + \epsilon/2 } \nonumber \\
	& = \P\braces{\exists t \leq M ~\colon~ L_t(\lambda) \geq e^{\lambda\epsilon/2} } \nonumber \\
	& \leq \E\braces{L_0(\lambda)}\cdot e^{-\lambda\epsilon/2} \nonumber \\
	& \leq e^{-\lambda\epsilon/2}.
	\end{align}
	
	It remains to choose $\lambda$ appropriately in order to get the tail bound we want. If $\epsilon \leq MK$, then we set $\lambda = \epsilon/2K^2M$, observing that for this choice of $\lambda$,
	\begin{equation*}
	\lambda \leq \frac{MK}{2K^2M} \leq \frac{1}{2K},
	\end{equation*}
	and $L_t(\lambda)$ is indeed a supermartingale, and \eqref{eq:supermartingale_argument} holds. Plugging our choice of $\lambda$ into the left hand and right hand sides, this yields the bound
	\begin{equation*}
	\P\braces{\exists t \leq M ~\colon~ S_t \geq \epsilon } \leq e^{-\epsilon^2/4mK^2}.
	\end{equation*}
	
	On the other hand, if $\epsilon > MK$, then we pick $\lambda = 1/2K$. Once again plugging this into \eqref{eq:supermartingale_argument}, we get
	\begin{equation*}
	\P\braces{\exists t \leq M ~\colon~ S_t \geq \epsilon } \leq e^{-\epsilon/4K}.
	\end{equation*}
	Putting these two bounds together gives the upper tail in \eqref{eq:uniform_Bernstein}, and considering the negative of the sequence gives the lower tail.
\end{proof}

\begin{lemma}[Chernoff for non-independent Bernoullis] \label{lem:Chernoff_for_martingale_difference_seq}
	Let $X_1,X_2,\ldots,X_M$ be a sequence of Bernoulli random variables adapted to a filtration $\braces{\mathcal{G}_t}$, and let $\theta$ be such that for $1 \leq t \leq M$, we have
	\begin{equation*}
	\E\braces{X_t~\vline~\mathcal{G}_{t-1}} \leq \theta
	\end{equation*}
	Then for any $0 < \epsilon < 1$, we have
	\begin{equation*}
	\P\braces*{\sum_{i=1}^M X_i \geq (1+\epsilon)M\theta} \leq e^{-M\theta\epsilon^2/3}
	\end{equation*}
\end{lemma}

\begin{proof}
	Denote $S_t \coloneqq \sum_{i=1}^t X_i$, observe that for any $\lambda > 0$,
	\begin{align*}
	\E\braces{e^{\lambda S_t}} & = \E\braces{\E\braces{e^{\lambda X_t} ~\vline~\mathcal{G}_{t-1} } e^{\lambda S_{t-1}}} \\
	& \leq \E\braces*{\exp\paren*{(e^\lambda-1)\E\braces{X_t~\vline~\mathcal{G}_{t-1}}} e^{\lambda S_{t-1}} } \\
	& \leq \exp\paren*{\theta(e^\lambda-1)} \E\braces{e^{\lambda S_{t-1}} },
	\end{align*}
	so that
	\begin{equation*}
	\E\braces{e^{\lambda S_M}} \leq \exp\paren*{M\theta(e^\lambda-1)}.
	\end{equation*}
	The rest of the proof is exactly the same as that of the regular Chernoff's inequality (see \cite{Vershynin}).
\end{proof}

\begin{lemma}[Moment bounds for contractions of random variables] \label{lem:4th_moment_of_contraction}
	Let $X$ be a random variable, $\rho\colon\R \to \R$ a contraction (i.e. a map such that $\abs{\rho(x)-\rho(y)} \leq \abs{x-y}$ for all $x$, $y$.) Then $\E\braces{\paren*{\rho(X)-\E\rho(X)}^4} \leq 8\E\braces{\paren*{X-\E X}^4}$. If in addition, $\rho$ is an odd function and $X$ is symmetric, then for any $b \in \R$, we also have $\abs*{\E\braces*{\rho(X + b)}} \leq \abs{b}$.
\end{lemma}

\begin{proof}
	Let $X'$ be an independent copy of $X$. By Jensen's inequality, followed by applying the contraction inequality pointwise, we have
	\begin{equation*}
	\E\braces{\paren*{\rho(X)-\E\rho(X)}^4} \leq \E\braces{\paren*{\rho(X)-\rho(X')}^4} \leq \E\braces*{\paren*{X-X'}^4}.
	\end{equation*}
	Writing $X-X'$ = $X - \E X + \E X' - X'$ and expanding, we can bound the right hand side via
	\begin{align*}
	\E\braces*{\paren*{X-X'}^4} & = \E\braces*{\paren*{X-\E X}^4} + 6\E\braces{\paren*{X-\E X}^2\paren*{X'-\E X'}^2} + \E\braces*{\paren*{X'-\E X'}^4} \\
	& \leq 8\E\braces{\paren*{X-\E X}^4}.
	\end{align*}
	Here, the inequality follows from Cauchy-Schwarz.
	
	For the second claim, assume $b > 0$ and write
	\begin{align*}
	\E\braces*{\rho(X + b)} & = \E\braces*{\rho(X + b) - \rho(X)} \\
	& \leq \E\braces*{\abs*{\rho(X + b) - \rho(X)}} \\
	& \leq \E\braces*{\abs*{(X+b) - X}} \\
	& = b.
	\end{align*}
	The first equality follows from the oddness of $\rho$ and the symmetry of $X$, while the second follows from contraction. If $b < 0$, the statement may be proved similarly.
\end{proof}

\end{document}